\documentclass{article}

\usepackage[in]{fullpage}  
\usepackage{natbib} 
\usepackage{mathpazo} 

\usepackage[utf8]{inputenc} %
\usepackage[T1]{fontenc}    %
\usepackage{hyperref}       %
\usepackage{url}            %
\usepackage{booktabs}       %
\usepackage{amsfonts}       %
\usepackage{nicefrac}       %
\usepackage{microtype}      %
\usepackage{lipsum}         %
\usepackage{todonotes}

\newcommand{\reals}{{\mathbb{R}}}

\newcommand{\diag}{\mathop{\bf diag}}

\newcommand{\rank}{\mathop{\bf rank}}

\newcommand{\argmin}{\mathop{\rm argmin}}
\newcommand{\argmax}{\mathop{\rm argmax}}

\newcommand{\dquote}[1]{``#1''}

\usepackage{url}            %
\usepackage{booktabs}       %
\usepackage{multirow}    
\usepackage{amsfonts}       %
\usepackage{nicefrac}       %
\usepackage{microtype}      %
\usepackage{natbib}
\usepackage{enumerate}
\usepackage{hhline}
\usepackage{makecell}
\usepackage{pifont}

\usepackage{graphicx} %
\usepackage{caption}
\usepackage{subcaption}
\usepackage{amsmath}
\usepackage{amsthm}
\usepackage{amssymb}
\usepackage{tikz}
\usepackage{xcolor}
\usetikzlibrary{arrows}

\allowdisplaybreaks

\usepackage{mathrsfs}

\usepackage{algorithm}
\usepackage{algpseudocode}
\usepackage{hyperref}
\usepackage{bm}

\allowdisplaybreaks

\newcommand{\lnorm}{\left\Vert}
\newcommand{\rnorm}{\right\Vert}

\newcommand{\tr}{\operatorname{tr}}

\newcommand{\supp}{\mathrm{supp}}

\newcommand{\real}{\mathbb{R}}

\def\TV{\mathrm{TV}}

\newcommand{\expect}{\mathbb{E}}

\newcommand{\indict}{\mathbb{I}}

\newtheorem{thm}{Theorem}
\newtheorem{lem}{Lemma}

\newtheorem{prop}{Proposition}
\newtheorem{asmp}{Assumption}

\newtheorem{rem}{Remark}
\newtheorem{example}{Example}

\usepackage[capitalize,noabbrev]{cleveref}
\crefname{thm}{Theorem}{Theorems}
\crefname{lem}{Lemma}{Lemmas}
\crefname{cor}{Corollary}{Corollaries}
\crefname{prop}{Proposition}{Propositions}
\crefname{asmp}{Assumption}{Assumptions}
\crefname{defn}{Definition}{Definitions}
\crefname{oracle}{Oracle}{Oracles}
\crefname{fact}{Fact}{Facts}
\crefname{conj}{Conjecture}{Conjectures}
\crefname{rem}{Remark}{Remarks}
\crefname{example}{Example}{Examples}
\crefname{condition}{Condition}{Conditions}
\crefname{exercise}{Exercise}{Exercises}
\crefname{algorithm}{Algorithm}{Algorithms}
\crefname{table}{Table}{Tables}
\crefname{figure}{Figure}{Figures}
\crefname{section}{Section}{Sections}
\crefname{subsection}{Section}{Sections}
\crefname{appendix}{Appendix}{Appendices}
\crefname{message}{Message}{Messages}

\definecolor{red}{rgb}{1, 0, 0}

\definecolor{green}{rgb}{0, 1, 0}

\definecolor{blue}{rgb}{0, 0, 1}

\definecolor{orange}{rgb}{1, 0.4, 0.0}

\input{packages/math_commands}

\usepackage{tabulary}

\usepackage{amssymb}%
\usepackage{pifont}%
\newcommand{\cmark}{\ding{51}}%
\newcommand{\xmark}{\ding{55}}%

\AtBeginDocument{\setlength\abovedisplayskip{5pt}}
\AtBeginDocument{\setlength\belowdisplayskip{5pt}}

\hypersetup{
    colorlinks=true,
    citecolor=blue,
    linkcolor=blue,
}

\newcommand{\bin}{\operatorname{Bin}}
\newcommand{\expert}{\operatorname{E}}
\newcommand{\bc}{\operatorname{BC}}
\newcommand{\piE}{\pi^{\expert}}

\newcommand{\gDE}{\gD^{\expert}}
\newcommand{\gDU}{\gD^{\operatorname{U}}}

\newcommand{\gDS}{\gD^{\operatorname{S}}}
\newcommand{\oU}{\operatorname{U}}

\newcommand{\oS}{\operatorname{S}}

\newcommand{\mix}{\operatorname{mix}}
\newcommand{\pimix}{\pi^{\operatorname{mix}}}

\newcommand{\tot}{\operatorname{tot}}
\newcommand{\pibcunion}{\pi^{\nbcu}}

\newcommand{\wbcu}{\operatorname{WBCU}}
\newcommand{\nbcu}{\operatorname{NBCU}}

\newcommand{\meanstd}[2]{$#1 {\scriptscriptstyle \pm #2}$}

\title{Theoretical Analysis of Offline Imitation With Supplementary Dataset}

\usepackage{authblk}

\author[1,2]{Ziniu Li \thanks{Equal contribution. Author ordering is determined by coin flip. Email: \texttt{ziniuli@link.cuhk.edu.cn} and \texttt{xut@lamda.nju.edu.cn}}}
\author[3]{Tian Xu{$^*$}}
\author[3]{Yang Yu\thanks{Corresponding author. Email: \texttt{yuy@nju.edu.cn} and \texttt{luozq@cuhk.edu.cn}}}
\author[1,2]{Zhi-Quan Luo{$^\dag$}}
\affil[1]{The Chinese University of Hong Kong, Shenzhen}
\affil[2]{Shenzhen Research Institute of Big Data}
\affil[3]{National Key Laboratory for Novel Software Technology, Nanjing University}

\date{\today}

\begin{document}

\maketitle
\begin{abstract}
Behavioral cloning (BC) can recover a good policy from abundant expert data, but may fail when expert data is insufficient. This paper considers a situation where, besides the small amount of expert data, a supplementary dataset is available, which can be collected cheaply from sub-optimal policies. Imitation learning with a supplementary dataset is an emergent practical framework, but its theoretical foundation remains under-developed. To advance understanding, we first investigate a direct extension of BC, called NBCU, that learns from the union of all available data. Our analysis shows that, although NBCU suffers an imitation gap that is larger than BC in the worst case, there exist special cases where NBCU performs better than or equally well as BC. This discovery implies that noisy data can also be helpful if utilized elaborately. Therefore, we further introduce a discriminator-based importance sampling technique to re-weight the supplementary data, proposing the WBCU method. With our newly developed \emph{landscape-based analysis}, we prove that WBCU can outperform BC in mild conditions. Empirical studies show that WBCU simultaneously achieves the best performance on two challenging tasks where prior state-of-the-art methods fail.
\end{abstract}

\section{Introduction}

Imitation learning (IL) methods train a good policy from expert demonstrations \citep{argall2009survey, osa2018survey}. One popular approach is behavioral cloning (BC) \citep{Pomerleau91bc}, which imitates the expert via supervised learning. Specifically, in IL, samples usually refer to state-action pairs/sequences from trajectories in the given dataset. Both the quality and quantity of trajectories are crucial to achieving satisfying performance. For instance, it is found that BC performs well when the dataset has a large amount of expert-level trajectories; see, e.g., \citep{spencer2021feedback}. Nevertheless, due to the compounding errors issue \citep{ross2010efficient}, any offline IL algorithm, including BC, will fail when the number of expert trajectories is small \citep{rajaraman2020fundamental, xu2021error}. To address the compounding errors issue, a naive solution is to ask the expert to collect more trajectories. However, querying the expert is expensive and intractable in applications such as healthcare and industrial control.

To address the mentioned failure mode, we follow an alternative and emergent framework proposed in \citep{kim2022demodice, xu2022discriminator}, which assumes, in addition to the expert dataset, a supplementary dataset is relatively cheap to obtain. In particular, this supplementary dataset could be previously collected by certain behavior policies; as such, it has lots of expert and sub-optimal trajectories. The key challenge here is to figure out which supplemental samples are helpful. We realize that a few advances have been achieved in this direction \citep{kim2022demodice, kim2022lobsdice, xu2022discriminator, ma2022smodice}. To leverage the supplementary dataset, a majority of algorithms train a discriminator to distinguish expert-style and sub-optimal samples, upon which a weighted BC objective is optimized to learn a good policy. For instance, \citet{kim2022demodice} proposed the DemoDICE algorithm, which learns the discriminator via a regularized state-action distribution matching objective. As for the DWBC algorithm in \citep{xu2022discriminator}, the policy and discriminator are jointly trained under a cooperation framework.

Though prior algorithms are empirically shown to perform well in some scenarios, the theoretical foundation of this new imitation problem remains under-developed. Specifically, researchers have a sketchy intuition that noisy demonstrations in the supplementary dataset may hurt the performance when we directly apply BC, and researchers hope to develop algorithms to overcome this challenge. However, the following questions have not been carefully answered yet: (\textbf{Q1}) precisely, what kind of noisy samples is harmful? (or what kind of supplementary samples is helpful?) \textbf{(Q2)} how to design algorithms in a principled way? Answers could deepen our understanding and provide insights for future advances.

In this paper, we explore the above questions by investigating two (independent) variants of BC. To learn a policy, both algorithms apply a BC objective on the \emph{union} dataset (the concatenation of the expert and supplementary dataset), but they differ in assigning weights of samples that appear in the loss function. To qualify the benefits of the supplementary dataset, we treat the BC \emph{only} with the expert dataset as a baseline to compare. 

The first algorithm, called NBCU (naive BC with union dataset), assigns uniform weights for all samples. As a direct extension of BC, the theoretical study of NBCU provides answers to (Q1). In particular, we show that NBCU suffers an imitation gap that is larger than BC in the \emph{worst} case (\cref{thm:bc_union} and \cref{prop:bc_union_lower_bound}), implying its inferior performance in the general case. However, we discover \emph{special} cases where NBCU performs better than or equally well as BC (\cref{thm:bc_imitation_gap_expert} and \cref{thm:bc_imitation_gap_no_overlap}). This discovery indicates that noisy data can also be helpful if utilized elaborately. To our best knowledge, such a message is new to the community. We provide the empirical evidence in \cref{fig:summary} and discuss these results in depth later.

As NBCU (i.e., the direct extension of BC) may fail in the worst case, we develop another algorithm called WBCU (weighted BC with union dataset) to address this failure mode. In light of \citep{kim2022demodice, xu2022discriminator}, WBCU trains a discriminator to weigh samples. We use the importance sampling technique \citep[Chapter 9]{shapiro2003monte} to design the weighting rule. Unlike prior works \citep{kim2022demodice, xu2022discriminator} that use a biased (or regularized) weighting rule, WBCU is theoretically sound and principled in the sense that the loss function is corrected as if samples from the supplementary dataset were collected by the expert policy; see \cref{rem:bias_issue} for detailed discussion.

Theoretically, we justify WBCU via a new landscaped-based analysis (\cref{thm:wbcu}). We characterize the landscape properties (e.g., the Lipschitz continuity and quadratic growth conditions) in \cref{lem:lipschitz_continuity} and \cref{lem:strong_convexity}.  Interestingly, we find that the \dquote{smooth} function approximation for the \emph{discriminator} plays a vital role in WBCU; without it, WBCU cannot be better than BC (\cref{prop:wbcu_negative}). Our theory can provide actionable guidance for practitioners; see \cref{subsec:positive_wbcu} for details. These results offer answers to (Q2).

\begin{figure}[htbp]
    \centering
    \includegraphics[width=0.5\linewidth]{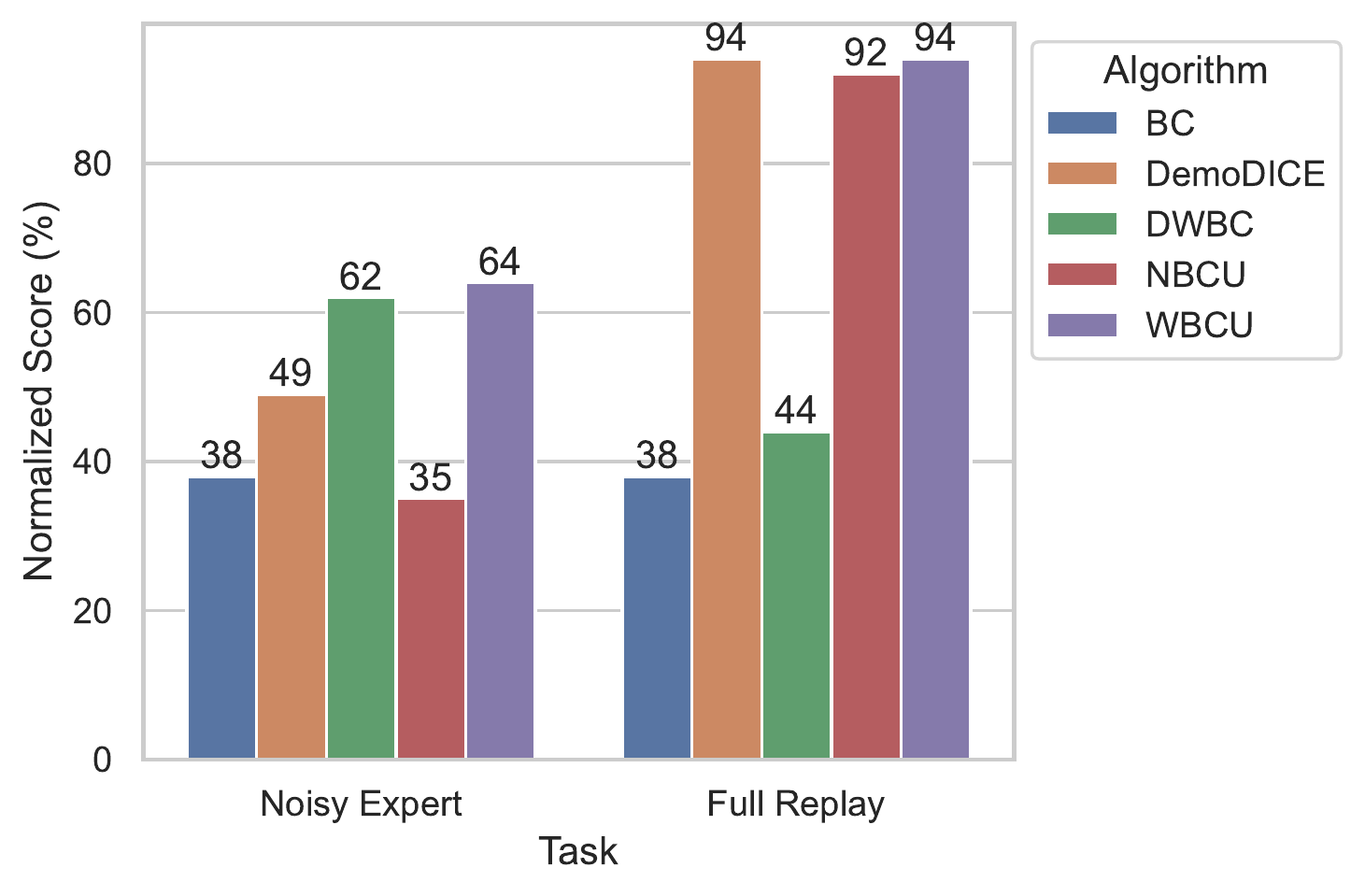}
    \caption{Averaged normalized scores of learned policies (over 4 MuJoCo environments) on two tasks with different supplementary datasets. A larger score means a better performance. Experiments show that 1) NBCU is worse than BC for the noisy expert task and better than BC for the full replay task; 2) only our method WBCU can perform well on both two tasks.}
    \label{fig:summary}
\end{figure}

Finally, we corroborate our claims with experiments on the MuJoCo locomotion control. Given the same expert dataset (1 expert trajectory), we consider two tasks with different supplementary datasets. The first task (called full replay) follows the setting in \citep{kim2022demodice}: supplementary trajectories are from the replay buffer of an online SAC agent \citep{haarnoja2018sac}. The second task (called noisy expert) is a new test-bed: the supplementary dataset contains clean expert-style trajectories, and the noisy counterparts\footnote{Noisy counterparts cover expert states but are injected with action noise.}. Please refer to \cref{appendix:experiment_details} for experiment details. Please see \cref{fig:summary} for the (averaged) normalized scores of learned policies over four representative MuJoCo environments\footnote{Ant-v2, HalfCheetah-v2, Hopper-v2, and Walker2d-v2.} (performance on individual environments is reported in \cref{tab:noisy_expert} and \cref{tab:full_replay} in Appendix). We find that WBCU \emph{simultaneously} achieves the best performance on both two tasks, whereas prior state-of-the-art methods like DemoDICE \citep{kim2022demodice} and DWBC \citep{xu2022discriminator} only perform well on one of two tasks. We will connect the experiment results with theoretical analysis in the main text.

\section{Related Work}

\textbf{Adversarial Imitation Learning.} Unlike BC, adversarial imitation learning (AIL) methods (e.g., GAIL \citep{ho2016gail}) do not suffer the compounding errors issue when the expert trajectories are limited; see the empirical evidence in \citep{ho2016gail, ghasemipour2019divergence, Kostrikov19dac} and theoretical support in \citep{xu2020error, xu2022understanding}. In particular, AIL methods train a discriminator to perform the state-action distribution matching, which differs from BC. Nevertheless, AIL methods naturally work in the online setting (i.e., the interaction is allowed). If the transition function is not available, BC is minimax optimal \citep{rajaraman2020fundamental, xu2021more} and offline counterparts of AIL are not better than BC \citep{li2022rethinking}.

Perhaps surprisingly, we find that the discriminator used in our WBCU algorithm plays a different role than AIL methods. We provide a detailed explanation in \cref{sec:analysis_of_weighted_behavioral_cloning}.

\textbf{Imitation Learning with Imperfect Demonstrations.} The problem considered in this paper is closely related to imitation learning (IL) with imperfect demonstrations \citep{wu2019imitation, brown2019extrapolating, tangkaratt2020variational, wang2021learning, sasaki2020behavioral, liu2022robust} in the sense that the supplementary dataset can also be viewed as imperfect demonstrations. However, our problem setting differs from IL with imperfect demonstrations in two key aspects. First, in IL with imperfect demonstrations, they either pose strong assumptions \citep{tangkaratt2020variational, sasaki2020behavioral, liu2022robust} or require auxiliary information (e.g., confidence scores on imperfect trajectories) on the imperfect dataset \citep{wu2019imitation, brown2019extrapolating}. In contrast, we assume access to a small number of expert trajectories, which could be more practical. Second, most works \citep{wu2019imitation, brown2019extrapolating, tangkaratt2020variational, wang2021learning} in IL with imperfect demonstrations require online environment interactions while we focus on the offline setting.

Finally, we point out that the problem of offline imitation learning with a supplementary dataset is first considered in \citep{kim2022demodice, xu2022discriminator}. Subsequently, \citet{ma2022smodice, kim2022lobsdice} study a related setting: learning from observation, where actions are missing and only states are observed. Existing works mainly focus on empirical studies, and the theoretical foundation remains under-developed.

\section{Preliminary}

\textbf{Markov Decision Process.} In this paper, we consider the episodic Markov decision process\footnote{Our results can be translated to the discounted and infinite-horizon MDP setting. We consider the episodic MDP because it allows a simple way of dealing with the statistical estimation.} (MDP)  $\gM = (\gS, \gA, \gP, r, H, \rho)$ \citep{puterman2014markov}. The first two elements $\gS$ and $\gA$ are the state and action space, respectively. $H$ is the maximum length of a trajectory and $\rho$ is the initial state distribution. $\gP = \{P_1, \cdots, P_{H}\}$ specifies the non-stationary transition function of this MDP; concretely, $P_h(s_{h+1}|s_h, a_h)$ determines the probability of transiting to state $s_{h+1}$ conditioned on state $s_h$ and action $a_h$ in time step $h$, for $h \in [H]$, where the symbol $[x]$ means the set of integers from $1$ to $x$. Similarly, $r = \{r_1, \cdots, r_{H}\}$ is the reward function, and we assume $r_h: \gS \times \gA \rar [0, 1]$, for $h \in [H]$. A time-dependent policy is $\pi_h: \gS \rar \Delta(\gA)$, where $\Delta(\gA)$ is the probability simplex. $\pi_h (a|s)$ gives the probability of selecting action $a$ on state $s$ in time step $h$, for $h \in [H]$. When the context is clear, we simply use $\pi$ to denote the collection of time-dependent policies $\{\pi_h\}_{h=1}^{H}$.

We measure the quality of a policy $\pi$ by the policy value (i.e., environment-specific long-term return): $V(\pi) = \expect\big[ \sum_{h=1}^{H} r(s_h, a_h) \mid s_1 \sim \rho; a_h \sim \pi_h (\cdot|s_h), s_{h+1} \sim P_h(\cdot|s_h, a_h), \forall h \in [H] \big]$. 
To facilitate later analysis, we need to introduce the state-action distribution $d^{\pi}_h(s, a)$: 
\begin{align*}
    d^{\pi}_h(s, a) = \sP\big( &s_h = s, a_h = a \mid s_1 \sim \rho; a_\ell \sim \pi(\cdot|s_\ell), s_{\ell+1} \sim P_\ell(\cdot|s_\ell, a_\ell), \forall \ell \in [h]  \big).
\end{align*}
Here, we use the convention that $d^{\pi}$ is the collection of all time-dependent state-action distributions. Sometimes, we need to consider the state distribution $d^{\pi}(s)$, which shares the same definition with $d^{\pi}(s, a)$ except that we only compute the probability of visiting a specific state $s$. It is easy to see that $d^{\pi}_h(s) = \sum_{a}d^{\pi}_h(s, a)$, $\forall h \in [H]$.

\textbf{Imitation Learning.} The goal of imitation learning is to learn a high quality policy directly from expert demonstrations. To this end, we often assume there is a nearly optimal expert policy $\piE$ that could interact with the environment to generate a dataset (i.e., $N_{\expert}$ trajectories of length $H$):
\begin{align*}
    \gD^{\expert} = \big\{ &\tr = \lp s_1, a_1, s_2, a_2, \cdots, s_H, a_H \rp; s_1 \sim \rho, a_h \sim \piE_h(\cdot|s_h), s_{h+1} \sim P_h(\cdot|s_h, a_h), \forall h \in [H] \big\}.
\end{align*}
Then, the learner can use $\gDE$ to mimic the expert. From a theoretical perspective, the quality of imitation is measured by the \emph{imitation gap}: $\expect\ls V (\piE) - V (\pi) \rs$, where $\pi$ is the learned policy and the expectation is taken over the randomness of $\gD^{\expert}$. We hope that a good learner can mimic the expert perfectly and thus the imitation gap is small. In this paper, we assume that the expert policy is deterministic, which is common in the literature \citep{rajaraman2020fundamental, rajaraman2021value, xu2020error, xu2021more}. Note that this assumption holds for tasks including MuJoCo locomotion control.

\textbf{Behavioral Cloning.} Given an \emph{expert} dataset $\gDE$, the behavioral cloning (BC) algorithm takes the maximum likelihood estimation:
\begin{align} \label{eq:bc_opt}
  \pi^{\bc} \in  \max_{\pi} \sum_{h=1}^{H} \sum_{ (s, a) \in \gS \times \gA} \widehat{d^{\expert}_h}(s, a) \log \pi_h(a|s),
\end{align}
where $\widehat{d^{\expert}_h}(s, a)$ is the empirical state-action distribution. By definition, $\widehat{d^{\expert}_h}(s, a) = n_h^{\expert}(s, a)/N_{\expert}$, where $n_h^{\expert}(s, a)$ refers to the number of expert trajectories such that their state-action pairs are equal to $(s, a)$ in time step $h$.  In the tabular case, a closed-formed solution to \cref{eq:bc_opt} is available:
\begin{align}   \label{eq:pi_bc}
\pi^{\bc}_h(a|s) = \left\{ \begin{array}{ll}
      \frac{n_h^{\expert}(s, a)}{n_h^{\expert}(s)}  & \text{if } n_h^{\expert}(s) > 0  \\
        \frac{1}{|\gA|} & \text{otherwise}  
    \end{array} \right.
\end{align}
where $n_h^{\expert}(s) \triangleq \sum_{a^\prime} n_h^{\expert}(s, a^\prime)$. For visited states, BC can make good decisions by duplicating expert actions. However, BC has limited knowledge about the expert actions on non-visited states. As a result, it may suffer a large imitation gap when making a wrong decision on a non-visited state, leading to compounding errors \citep{ross2010efficient}.

\textbf{Imitation with Supplementary Dataset.} BC will fail when the number of expert trajectories is small due to the mentioned compounding errors issue. A naive solution is to collect more expert trajectories. In this paper, we follow an alternative and emergent framework in \citep{kim2022demodice, xu2022discriminator}, where we assume that an offline \emph{supplementary} dataset $\gDS$ (i.e., $N_{\oS}$ trajectories of length $H$) is relatively cheap to obtain:
\begin{align*}
    \gDS = \big\{ &\tr = \lp s_1, a_1, s_2, a_2, \cdots, s_H, a_H \rp; s_1 \sim \rho,  a_h \sim \pi^{\beta}_h(\cdot|s_h), s_{h+1} \sim P_h(\cdot|s_h, a_h), \forall h \in [H] \big\},
\end{align*}
where $\pi^{\beta}$ is a behavioral policy that could be a mixture of certain base policies, i.e., $\pi^{\beta} = \sum_{i=1}^{K} \alpha_i \pi^i$ with $\sum_{i=1}^{K} \alpha_i = 1$ and $\alpha_i \geq 0$ for $i \in [K]$. Algorithms can additionally leverage this supplementary dataset to mitigate the compounding errors issue.

\section{Analysis of Naive Behavior Cloning with Union Dataset}
\label{sec:analysis_of_behavioral_cloning}

In this section, we consider a direct extension of BC for the problem of offline imitation with a supplementary dataset. Specifically, this algorithm called NBCU (naive BC with the union dataset) performs the maximum likelihood estimation on the union of the expert and supplementary datasets: 
\begin{align} \label{eq:bc_union}
  \pi^{\nbcu} \in  \argmax_{\pi} \sum_{h=1}^{H} \sum_{(s, a)} \widehat{d^{\oU}_h}(s, a) \log \pi_h(a|s),
\end{align}
where $\gDU = \gDS \cup \gDE$ and $\widehat{d^{\oU}_h}(s ,a)$ is the empirical state-action distribution estimated from $\gDU_h$ (i.e., the subset of $\gDU$ in step $h$). As an analogue to \cref{eq:pi_bc}, we have 
\begin{align}   \label{eq:pi_nbcu}
\pi^{\nbcu}_h(a|s) = \left\{ \begin{array}{ll}
      \frac{n_h^{\oU}(s, a)}{n_h^{\oU}(s)}  & \text{if } n_h^{\oU}(s) > 0  \\
        \frac{1}{|\gA|} & \text{otherwise}  
    \end{array} \right.
\end{align}
where, just like before, $n_h^{\oU}(s, a)$ refers to the number of union trajectories such that their state-action pairs are equal to $(s, a)$ in time step $h$. To analyze NBCU, we pose an assumption about the dataset collection.

\begin{algorithm}[htbp]
\caption{NBCU}
\label{algo:nbcu}
\begin{algorithmic}[1]
\Require{Expert dataset $\gDE$ and supplementary dataset $\gDS$.}
\State{$\gDU \lar \gDE \cup \gDS$.}
\State{Apply BC to learn a policy $\pi$ by objective \eqref{eq:bc_union} with $\gDU$.}
\end{algorithmic}
\end{algorithm}

\begin{asmp}   \label{asmp:dataset_collection}
The supplementary dataset $\gDS$ and expert dataset $\gDE$ are collected in the following way: each time, we roll-out a behavior policy $\pi^{\beta}$ with probability $1-\eta$ and the expert policy with probability $\eta$, where $\eta \in [0, 1]$ controls the fraction of expert trajectories. Such an experiment is independent and identically conducted by $N_{\tot}$ times. 
\end{asmp}

Under \cref{asmp:dataset_collection}, we slightly overload our notations: let $N_{\expert}$ be the \emph{expected} number of expert trajectories,  i.e., $ N_{\expert} = \eta N_{\tot}$, and $N_{\oS}$ be the \emph{expected} number of supplementary trajectories, i.e.,  $N_{\oS} = (1-\eta) N_{\tot}$.

\subsection{Baseline: BC on the Expert Dataset}
\label{subsec:baseline:bc_expert}

To measure whether the supplementary dataset is helpful, we consider the BC \emph{only} with the expert dataset as a baseline. This approach has been analyzed in \citep{rajaraman2020fundamental}, and we transfer their results under our assumption\footnote{The proof of \cref{thm:bc_expert} builds on \citep{rajaraman2020fundamental} and the main difference is that the number of expert trajectories is a random variable in our set-up. Technically, we handle this difficulty by \cref{lem:binomial_distribution} in Appendix.}: 
\begin{thm}
\label{thm:bc_expert}
Under \cref{asmp:dataset_collection}. In the tabular case, if we apply BC only on the expert dataset, we have that\footnote{$a(n) \lesssim b(n)$ means that there exist $C, n_0 > 0$ such that $a(n) \leq C b(n)$ for all $n \geq n_0$. In our context, $n$ usually refers to the number of trajectories.} 
\begin{align*}
    \expect \ls V(\piE) - V(\pi^{\bc}) \rs \lesssim \min \lb H, \frac{|\gS| H^2}{N_{\expert}} \rb,
\end{align*}
where the expectation is taken over the randomness in the dataset collection. 
\end{thm}

Proofs of all theoretical results are deferred to the Appendix.

\subsection{Main Results of NBCU}

Now, we present our main claim about NBCU.
\begin{thm} \label{thm:bc_union}
Under \cref{asmp:dataset_collection}. In the tabular case, for any $\eta \in (0, 1]$, we have 
\begin{align*}  
    &\expect \ls V(\piE) - V(\pi^{\nbcu}) \rs  \lesssim \min \bigg\{ H,  (1-\eta) \lp V (\piE) - V (\pi^{\beta}) \rp + \frac{|\gS| H^2 \log (N_{\tot})}{N_{\tot}} \bigg\}, 
\end{align*}
where the expectation is taken over the randomness in the dataset collection. 
\end{thm}

We often have $V(\piE) - V(\pi^{\beta}) > 0$, as the behavior policy is usually inferior to the expert policy. In this case, even if $N_{\tot}$ is sufficiently large so that the second term is negligible, there still exists a positive gap between $V(\piE)$ and $V(\pi^{\nbcu})$. Fundamentally, this is because the behavior policy may collect non-expert actions, so the recovered policy may select a wrong action even on expert states, which results in bad performance. The following theorem shows that the gap $V(\piE) - V(\pi^{\beta})$ is inevitable in the worst case.

\begin{prop}
\label{prop:bc_union_lower_bound}
Under \cref{asmp:dataset_collection}. In the tabular case, there exists an MDP $\gM$, an expert policy $\piE$ and a behavior policy $\pi^\beta$, for any $\eta \in (0, 1]$, when $N_{\tot} \gtrsim |\gS|$, we have 
\begin{align*}
    \expect \ls V(\piE) - V(\pi^{\nbcu}) \rs  \gtrsim  (1-\eta) \lp V (\piE) - V (\pi^{\beta}) \rp, 
\end{align*}
where the expectation is taken over the randomness in the dataset collection.
\end{prop}

The construction of the hard instance in \cref{prop:bc_union_lower_bound} is based on the following intuition: NBCU does not distinguish the action labels in the union dataset and treats them equally important; see \cref{eq:pi_nbcu}. Therefore, NBCU will learn bad decisions and suffer a non-vanishing gap when the dataset has a bad state-action coverage (i.e., primarily, action labels on the expert states are sub-optimal).

\subsection{Positive Results of NBCU}

The previous results suggest that NBCU is not warranted to be better than BC. For some special cases (depending on the dataset coverage), however, we can bypass the hard instance in \cref{prop:bc_union_lower_bound} and show that NBCU can perform well. We state two representative results as follows. 

\begin{thm}  \label{thm:bc_imitation_gap_expert}
Under the same assumption with \cref{thm:bc_union}, additionally assume that $\pi^{\beta} = \pi^{\expert}$. In the tabular case, for any $\eta \in [0, 1]$, we have 
\begin{align*}  
    \expect \ls V(\piE) - V(\pi^{\nbcu}) \rs \lesssim \min \lb H, \frac{|\gS| H^2}{N_{\tot}} \rb, 
\end{align*}
where the expectation is taken over the randomness in the dataset collection. 
\end{thm}

\cref{thm:bc_imitation_gap_expert} is a direct extension of \cref{thm:bc_expert}. The condition $\pi^{\beta} = \piE$ seems strong. Still, it may approximately hold in the following case: the supplementary dataset $\gDS$ is collected by an online agent that improves its performance over iterations, in which $\pi^{\beta}$ is close to $\piE$ asymptotically.

\begin{thm}   \label{thm:bc_imitation_gap_no_overlap}
Under the same assumption with \cref{thm:bc_union}, additionally assume that $\pi^{\beta}$ never visits expert states, i.e., $\supp(d^{\pi^{\beta}}_h(\cdot)) \cap \supp(d^{\pi^{\expert}}_h(\cdot)) = \emptyset$\footnote{For a distribution $p$, $\supp(p) = \{ x: p(x) \ne 0 \}$.} for all $h \in [H]$. In the tabular case, for any $\eta \in (0, 1]$, we have 
\begin{align}   \label{eq:bc_imitation_gap_no_overlap}
    \expect \ls V(\piE) - V(\pi^{\nbcu}) \rs \lesssim \min \lb H, \frac{|\gS| H^2}{ N_{\expert} } \rb, 
\end{align}
where the expectation is taken over the randomness in the dataset collection. 
\end{thm}

It is commonly believed that noisy demonstrations are harmful to performance \citep{sasaki2020behavioral}. Nevertheless, \cref{thm:bc_imitation_gap_no_overlap} provides a condition in which BC is robust to noisy samples. Technically, this robustness property stems from the theoretical analysis that only states along the expert trajectory are significant for the imitation gap, and noisy actions on the non-expert states do not contribute.

\textbf{Summary.} Our results demonstrate that without a special design, the naive application of BC on the union dataset is not guaranteed to be better than BC. However, good results may happen if the dataset coverage is nice, suggesting that noisy demonstrations are not the monster in all cases. Please refer to \cref{tab:nbcu_summary} for a summary.

\begin{table}[htbp]
    \caption{Effects of state-action attributes of samples for NBCU. \dquote{\cmark\cmark} indicates a helpful sample, \dquote{\xmark} indicates a harmful sample, and \dquote{\cmark} means something in between.}
    \label{tab:nbcu_summary}
    \centering
    \begin{tabular}{c|c|c}
                        &  Expert state & Non-expert state  \\  \hline 
       Expert action     &  \cmark\cmark &   \cmark  \\
       Non-expert action &  \xmark & \cmark
    \end{tabular}
\end{table}

\textbf{Connection with Experiments.} From \cref{fig:summary}, we already see two interesting phenomena: 1) NBCU is worse than BC for the noisy expert task; 2) NBCU is much better than BC for the full replay task. We use our theory to interpret these results. First, for the noisy expert task, our experiment setting (refer to \cref{appendix:experiment_details} for details) ensures that the supplementary dataset has noisy demonstrations on expert states, following the idea in \cref{prop:bc_union_lower_bound}. In this case, \cref{thm:bc_union} predicts that NBCU is no better than BC. Second, for the full replay task, we remark that the replay buffer contains lots of expert-level trajectories (refer to \cref{fig:sac} in Appendix which shows that the online SAC converges to the expert-level performance quickly). Therefore, the dataset coverage is good in this scenario and \cref{thm:bc_imitation_gap_expert} and \cref{thm:bc_imitation_gap_no_overlap} can explain the good performance of NBCU.

\section{Analysis of Weighted Behavioral Cloning with Union Dataset}
\label{sec:analysis_of_weighted_behavioral_cloning}

In this section, we explore an alternative approach to leverage the expert and supplementary datasets. In light of \citep{kim2022demodice, xu2022discriminator}, a discriminator is trained to score samples, upon which a weighted BC objective is used for policy optimization:
\begin{align}   \label{eq:weighted_bc}
   &\pi^{\wbcu} \in  \argmax_{\pi} \sum_{h=1}^{H} \sum_{(s, a) \in \gS \times \gA} \bigg\{ \widehat{d^{\oU}_h}(s, a)  \times \ls  w_h(s, a) \log \pi_h(a|s) \rs  \times  \sI \ls   w_h(s, a) \geq \delta \rs  \bigg\}, 
\end{align}
where $\widehat{d^{\oU}_h}(s, a)$ is the empirical state-action distribution estimated from $\gDU_h$, $w_h(s, a) \in [0, \infty)$ is the weight decided by the discriminator, and $\delta \in [0, \infty)$ is a hyper-parameter. We point out that $\delta$ is introduced for theoretical analysis, and in practice we set $\delta = 0$. We call this approach WBCU (weighted BC with the union dataset).

Our key idea is the importance sampling technique \citep[Chapter 9]{shapiro2003monte}, which can transfer samples in the union dataset under the expert policy distribution. In this way, WBCU is expected to address the failure mode of NBCU. We elaborate on this point as follows. In the population level (i.e., there are infinitely samples), we would have $\widehat{d^{\oU}_h} = d^{\oU}_h$, which is jointly determined by the expert policy and the behavioral policy. In this case, if $w_h(s, a) = d^{\expert}_h(s, a) / d^{\oU}_h(s, a)$, we would have $\widehat{d^{\oU}_h}(s, a) w_h(s, a) = d^{\expert}_h(s, a)$. Accordingly, the objective \eqref{eq:weighted_bc} is to learn a policy as if samples were collected by the expert policy. In practice, $d^{\expert}_h(s, a)$ and $d^{\oU}_h(s, a)$ are unknown; instead, we only have finite samples from these two distributions. Therefore, we need to estimate the grounded importance sampling ratio $d^{\expert}_h(s, a) / d^{\oU}_h(s, a)$.

We emphasize that a simple two-step idea---first estimating $d^{\expert}_h(s, a)$ and $d^{\oU}_h(s, a)$ separately and then calculating their quotient---is intractable. This is because it is difficult to accurately estimate the probability density of high-dimensional distributions. Following the seminal idea in \citep{goodfellow2014gan}, we take a one-step approach: we directly train a discriminator to estimate $d^{\expert}_h(s, a) / d^{\oU}_h(s, a)$. Concretely, we consider time-dependent parameterized discriminators $\{ c_h : \gS \times \gA \rar (0, 1)\}_{h=1}^{H}$, each of which has an objective
\begin{align}  
  \max_{c_h} & \sum_{(s, a) \in \gS \times \gA} \widehat{d^{\expert}_h}(s, a) \ls \log \lp c_h (s, a) \rp \rs  + \sum_{(s, a) \in \gS \times \gA} \widehat{d^{\oU}_h} (s, a) \ls \log \lp 1- c_h (s, a) \rp \rs.      \label{eq:discriminator_opt}
\end{align}
The above problem amounts to training a binary classifier (i.e., the logistic regression). In the population level, with the first-order optimality condition, we have 
\begin{align}   \label{eq:closed_form}
    c^\star_h (s, a) = \frac{{d^{\expert}_h} (s, a)}{ {d^{\expert}_h} (s, a) + {d^{\oU}_h} (s, a)}.
\end{align}
Then, we can obtain the importance sampling ratio $d^{\expert}_h(s, a) / d^{\oU}_h(s, a)$ in the following way:
\begin{align}   \label{eq:from_c_to_w}
    w_h(s, a) = \frac{c^{\star}_h(s, a)}{1 - c^{\star}_h(s, a)}.
\end{align}
Based on the above discussion, we outline the implementation of the proposed method WBCU in \cref{algo:wbcu}.

\begin{algorithm}[htbp]
\caption{WBCU}
\label{algo:wbcu}
\begin{algorithmic}[1]
\Require{Expert dataset $\gDE$ and supplementary dataset $\gDS$.}
\State{$\gDU \lar \gDE \cup \gDS$.}
\State{Train a binary classifier $c$ by objective \eqref{eq:discriminator_opt} with $\gDE$ and $\gDU$.}
\State{Compute the importance sampling ratio $w$ by \cref{eq:from_c_to_w}.}
\State{Apply BC to learn a policy $\pi$ by objective \eqref{eq:weighted_bc} with $\gDU$.}
\end{algorithmic}
\end{algorithm}

\begin{rem}  \label{rem:bias_issue}

The weighting rule of WBCU is unbiased in the sense that WBCU directly estimates the importance sampling ratio, while prior methods in \citep{kim2022demodice, xu2022discriminator} use biased/regularized weighting rules. As byproducts, WBCU has fewer hyper-parameters to tune.

First, DemoDICE also implements the policy learning objective \eqref{eq:weighted_bc}, but DemoDICE uses the weighting rule $\widetilde{w}(s, a) \propto d^{\star}(s, a) / d^{\oU}(s, a)$ (refer to the formula between Equations (19)-(20) in \citep{kim2022demodice}), where $d^{\star}$ is computed by a regularized state-action distribution objective (refer to \citep[Equations (5)-(7)]{kim2022demodice})\footnote{For a moment, we use the notations in \citep{kim2022demodice} and present their results under the stationary and infinite-horizon MDPs. Same as the discussion of DWBC \citep{xu2022discriminator}.}:
\begin{align*}
     d^{\star} &= \argmin_{d}   \KL(d \Vert d^{\expert}) + \alpha \KL(d \Vert d^{\oU}) \\
     \text{s.t.}  \quad & d(s, a) \geq 0 \quad \forall s, a.  \\
    & \sum_{a} d(s, a) = (1-\gamma) \rho(s) + \gamma \sum_{s^\prime, a^\prime} P(s|s^\prime, a^\prime) d(s^\prime, a^\prime) \quad \forall s. 
\end{align*}
where $\gamma \in [0, 1)$ is the discount factor, $\alpha > 0$ is a hyper-parameter. Due to the regularized term in the objective and the Bellman flow constraint, we have $d^{\star} \ne d^{\expert}$.

Second, DWBC considers a different policy learning objective (refer to \citep[Equation (17)]{xu2022discriminator}):
\begin{align}
    \min_{\pi} \quad &\alpha \sum_{(s, a) \in \gD^{\expert}} \ls - \log \pi(a|s) \rs - \sum_{(s, a) \in \gDE} \ls - \log \pi(a|s) \cdot \frac{\lambda}{c( 1- c)} \rs  + \sum_{(s, a) \in \gDS} \ls - \log \pi(a|s) \cdot \frac{1}{1-c} \rs,   \label{eq:dwbc}
\end{align}
where $\alpha > 0, \lambda > 0$ are hyper-parameters, and $c$ is the output of the discriminator that is jointly trained with $\pi$ (refer to \citep[Equation (8)]{xu2022discriminator}):
\begin{align*}
    \min_{c} \quad & \lambda \sum_{(s, a) \in \gDE} \ls - \log c(s, a, \log \pi(a|s)) \rs  + \sum_{(s, a) \in \gDS} \ls - \log (1 - c(s, a, \log \pi(a|s))) \rs  \\
    &- \lambda \sum_{(s, a) \in \gDE} \ls - \log (1 - c(s, a, \log \pi(a|s)))\rs. 
\end{align*}
Since its input additionally incorporates $\log \pi$, the discriminator is not guaranteed to estimate the state-action distribution. Thus, the weighting in \cref{eq:dwbc} loses a connection with the importance sampling ratio. 
\end{rem}

Experiments in \cref{fig:summary} show that WBCU simultaneously works well on the noisy expert and full replay tasks, while prior methods like DemoDICE and DWBC perform well only on one of them. We believe the discussion in \cref{rem:bias_issue} can partially explain the empirical observation. Next, we investigate the theoretical foundation of WBCU.

\subsection{Negative Result of WBCU With Tabular Representation}

In this section, we consider parameterizing the discriminator $c$ by a huge table (i.e., a vector with dimension $|\gS||\gA|$). We present a surprising counter-intuitive result. 

\begin{prop}    \label{prop:wbcu_negative}
In the tabular case (with $\delta = 0$),  we have $\pi^{\wbcu} = \pi^{\bc}$. 
\end{prop}

\cref{prop:wbcu_negative} shows that even if we have a large number of supplementary samples and even if we use the importance sampling, WBCU is not guaranteed to outperform BC, based on the tabular representation. We highlight that this failure mode is because the discriminator has no extrapolation ability in this case. Specifically, for an expert-style sample $(s, a)$ that \emph{only} appears in the supplementary dataset, we have $\widehat{d^{\expert}_h}(s, a) = 0$ and $\widehat{d^{\oU}_h}(s, a) > 0$, so $c^{\star}_h(s, a) = \widehat{d^{\expert}_h}(s, a) / (\widehat{d^{\expert}_h}(s, a) + \widehat{d^{\oU}_h}(s, a)) = 0$. That is, such a good sample does not contribute to the learning objective \eqref{eq:weighted_bc}. Intuitively, the tabular representation simply treats samples in a discrete way, which ignores the correlation between samples.

\begin{figure*}[t]
    \centering
    \includegraphics[width=0.8\linewidth]{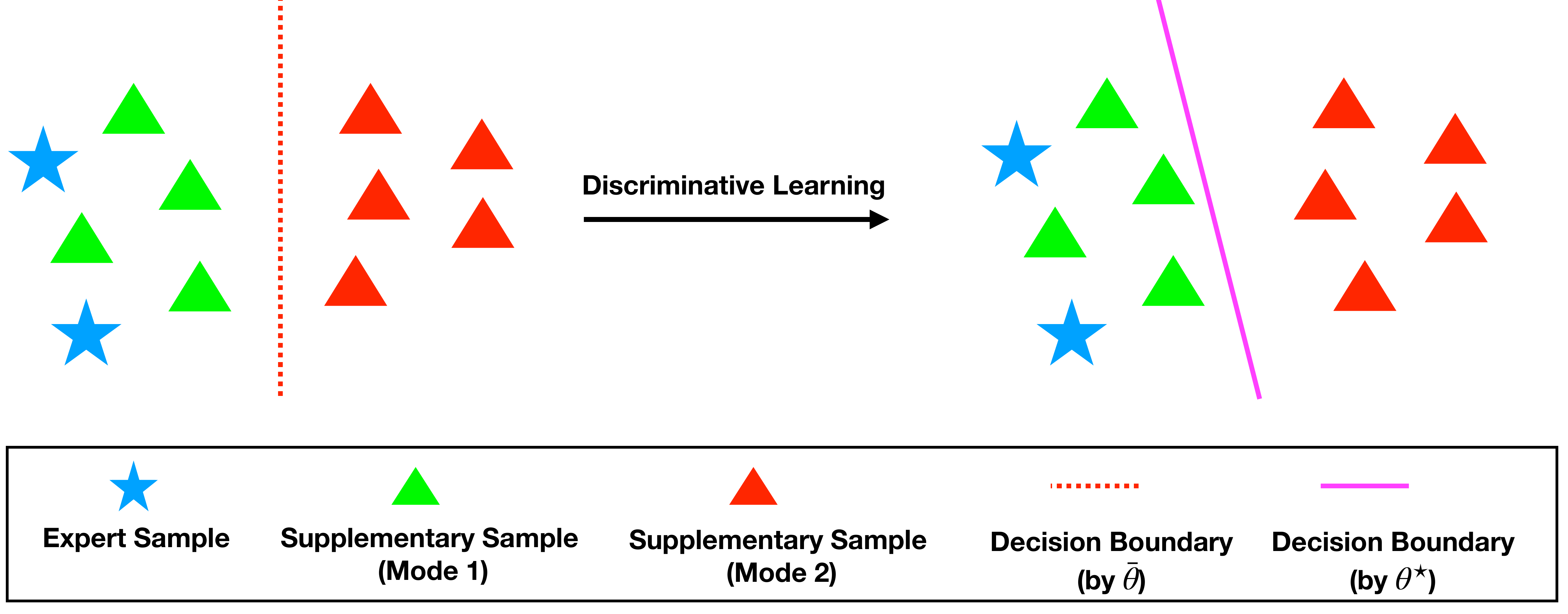}
    \caption{Illustration for the learning scheme of WBCU under \cref{asmp:linear_separable}.}
    \label{fig:wbcu}
\end{figure*}

\subsection{Positive Result of WBCU With Function Approximation}
\label{subsec:positive_wbcu}

To bypass the hurdle in the previous section, we investigate WBCU with certain function approximation in this section. To avoid the tabular/discrete representation, we will consider \dquote{smooth} function approximators, which can model the internal correlation between samples. Specifically, we consider the discriminator to be parameterized by 
\begin{align}   \label{eq:discriminator}
    c_h(s, a; \theta_h) = \frac{1}{1 + \exp ( - \langle \phi_h(s, a), \theta_h \rangle)},
\end{align}
where $\phi_h(s, a) \in \real^{d}$ is the feature vector (that can be learned by neural networks), and $\theta_h \in \real^{d}$ is the parameter to train. Accordingly, the optimization problem becomes:
\begin{align}  
  & \min_{\theta_h}  \gL_h (\theta_h) \triangleq \bigg\{ \sum_{(s, a)} \widehat{d^{\expert}_h}(s, a) \ls \log \lp 1 + \exp \lp - \langle \phi_h (s,a), \theta_h \rangle \rp \rp \rs  + \sum_{(s, a)} \widehat{d^{\oU}_h} (s, a) \ls \log \lp 1 + \exp \lp \langle \phi_h (s,a), \theta_h \rangle \rp \rp \rs \bigg\}.      \label{eq:parameterized_discriminator_opt}
\end{align}
Let $\theta^{\star} = \{ \theta_1^{\star}, \cdots, \theta_H^{\star}\}$ be the optimal solution obtained from \cref{eq:parameterized_discriminator_opt}. Due to the side information in the feature, samples are no longer treated independently, and the discriminator can perform a structured estimation. We clarify that to be consistent with the previous results, the policy is still based on the tabular representation. We discuss the general function approximation of policy in \cref{sec:bc_gfa}.

Since $c^{\star}$ is no longer analytic as in \cref{eq:closed_form}, a natural question is: what can we say about it? Our intuition is stated as follows. Let $\gDS_h$ denote the set of samples in step $h$ in $\gDS$. Since the behavior policy that collects $\gDS_h$ is diverse, we can imagine $\gDS_h$ contains two modes of samples: some of them actually are also collected by the expert policy while others are not. In the former case, we expect that $w_h(s, a)$ is large so that $\widehat{d^{\oU}_h}(s, a) w_h(s, a) \approx 1$, indicating such a sample $(s, a)$ is likely collected by the expert. In the latter case, we hope the discriminator can predict $w_h$ with a small value so that $\widehat{d^{\oU}_h}(s, a) w_h(s, a) \approx 0$, indicating it is a non-expert sample. Notice that $w_h$ is monotone with respect to the inner product $\langle \phi_h(s, a), \theta \rangle$; refer to Equations \eqref{eq:from_c_to_w}\eqref{eq:discriminator}. Therefore, we conclude that a larger $\langle \phi_h(s, a), \theta \rangle$ means a significant contribution to the learning objective \eqref{eq:weighted_bc}. Next, we demonstrate that the above intuition can be achieved under mild assumptions.

\begin{asmp}[Linear Separability]    \label{asmp:linear_separable}
Let $\gDS = \gD^{\oS, 1} \cup \gD^{\oS, 2}$ and $\gD^{\oS, 1} \cap \gD^{\oS, 2} = \emptyset$, where $\gD^{\oS, 1}$ is collected by the expert policy and $\gD^{\oS, 2}$ is collected by a sub-optimal policy (but the algorithm does not know this split). For each time step $h \in [H]$, there exists a ground truth parameter $\bar{\theta}_h \in \reals^d$, for any $(s, a) \in \gDE_h \cup \gD^{\oS, 1}_h$ and $(s^\prime, a^\prime) \in \gD^{\oS, 2}_h$, it holds that
\begin{align*}
    \langle  \bar{\theta}_h,  \phi_h (s, a) \rangle  > 0,  \langle  \bar{\theta}_h,  \phi_h (s^\prime, a^\prime) \rangle  < 0.
\end{align*}
\end{asmp}

Readers may realize that \cref{asmp:linear_separable} is closely related to the notion of \dquote{margin} in the classification problem. Define 
\begin{align*}
    \Delta_h(\theta) \triangleq \bigg\{  & \min_{(s, a) \in \gDE_h \cup \gD^{\oS, 1}_h}   \langle  \theta,  \phi_h (s, a) \rangle - \max_{(s^\prime, a^\prime) \in \gD^{\oS, 2}_h}  \langle  \theta,  \phi_h (s^\prime, a^\prime) \rangle  \bigg\}.
\end{align*}
From \cref{asmp:linear_separable}, we have $\Delta_h(\bar{\theta}_h) > 0$. This means that there \emph{exists} a classifier that recognizes samples from both $\gDE_h$ and $\gD^{\oS, 1}_h$ as \dquote{good} samples, which contributes to the objective \eqref{eq:weighted_bc}. On the other hand, samples from $\gD^{\oS, 2}_h$ will be classified as \dquote{bad} samples, which do not matter for the learned policy. Note that such a nice classifier is assumed to exist, which is not identical to what is learned via \cref{eq:parameterized_discriminator_opt}. Next, we theoretically control the (parameter) distance between two classifiers.

Before further discussion, we note that $\bar{\theta}_h$ is not unique if it exists. Without loss of generality, we define $\bar{\theta}_h$ as that can achieve the maximum margin (among all unit vectors\footnote{Otherwise, the margin is unbounded by multiplying $\bar{\theta}_h$ with a positive scalar.}). To theoretically characterize the movement of $\theta^{\star}$, we first characterize the landscape properties (e.g., Lipschitz continuity and quadratic growth conditions\footnote{These terminologies are from the optimization literature (see, e.g., \citep{karimi2016linear, drusvyatskiy2018error}).}) of $\Delta_h$ and $\gL_h(\theta)$ in \cref{lem:lipschitz_continuity} and \cref{lem:strong_convexity}, respectively.

\begin{lem}[Lipschitz Continuity]   \label{lem:lipschitz_continuity}
For any $\theta \in \real^{d}$, the margin function is $L_{h}$-Lipschitz continuous in the sense that 
\begin{align*}
 \Delta_h(\bar{\theta}_h)  - \Delta_h(\theta) \leq L_h \lnorm \bar{\theta}_h - \theta \rnorm ,
\end{align*}
where $L_h = \lnorm \phi_h (s^1, a^1) - \phi_h (s^2, a^2) \rnorm$ with $(s^1, a^1) \in \argmin_{ (s, a) \in \gDE_h \cup \gD^{\oS, 1}_h} \langle \theta,  \phi_h (s, a) \rangle$ and $(s^2, a^2) \in \\ \argmax_{(s, a) \in  \gD^{\oS, 2}_h } \langle \theta,  \phi_h (s, a) \rangle$.
\end{lem}

\begin{lem}[Quadratic Growth]  \label{lem:strong_convexity}
For any $h$, let $A_h \in \reals^{N_{\tot} \times d}$ be the matrix that aggregates the feature vectors of samples in $\gDU_h$. Consider the under-parameterization case that $\rank (A_h) = d$. There exists $\tau_h > 0$ such that
\begin{align*}
    \gL_h (\widebar{\theta}_h) \geq \gL_h (\theta^\star_h) + \frac{\tau_h}{2} \lnorm \widebar{\theta}_h - \theta^\star_h \rnorm^2.
\end{align*}
\end{lem}

\begin{thm}  \label{thm:wbcu}
Under \cref{asmp:linear_separable}, for any $h \in [H]$, if the following inequality holds 
\begin{align}   \label{eq:condition}
   \sqrt{ \frac{2 \lp \gL_h(\bar{\theta}_h) - \gL_h(\theta^{\star}_h) \rp}{\tau_h}} < \frac{\Delta_h (\bar{\theta}_h)}{L_h},
\end{align}
then we have $\Delta_h(\theta^{\star}_h) > 0$. 
\end{thm}

To interpret \cref{thm:wbcu}, we note that $\Delta_h(\theta^{\star}_h) > 0$ means that the learned discriminator can perfectly distinguish the good samples (from $\gDE_h$ and $\gD^{\oS, 1}_h$) and bad samples (from $\gD^{\oS, 2}_h$). In other words, if the feature design is nice such that Inequality \eqref{eq:condition} holds, then the obtained classifier can still maintain the decision results by $\bar{\theta}$; refer to \cref{fig:wbcu} for illustration. Consequently, all samples from $\gDE_h$ and $\gD^{\oS, 1}_h$ are assigned large weights. In this way, WBCU can leverage additional samples to outperform BC. Technically, $\Delta_h(\theta_h^{\star}) > 0$ means that there exists a $\delta > 0$ such that we have $w_h(s, a) > \delta$ for $(s, a) \in \gDE_h \cup \gD^{\oS, 1}_h$ and $w_h(s, a) < \delta$ for $(s, a) \in \gD^{\oS, 2}_h$. As such, WBCU can utilize the good samples $\gD^{\oS, 1}_h$ and eliminate the bad samples $\gD^{\oS, 2}_h$ in theory. The detailed imitation gap bound depends on the number of trajectories in $\gDE \cup \gD^{\oS, 1}$, and this result is straightforward, so we omit details here.

Readers may ask whether inequality \eqref{eq:condition} can hold in practice. This question is hard to answer because the coefficient $\tau_h$ and $L_h$ are data-dependent. Nevertheless, we can provide a toy example to illustrate that inequality \eqref{eq:condition} holds and give a sharp condition for $d= 1$; please refer to \cref{sec:proof_of_weighted_behavioral_cloning} for details. Further relaxation of the condition and assumption is left for future work.

\begin{figure}[htbp]
    \centering
    \includegraphics[width=0.5\linewidth]{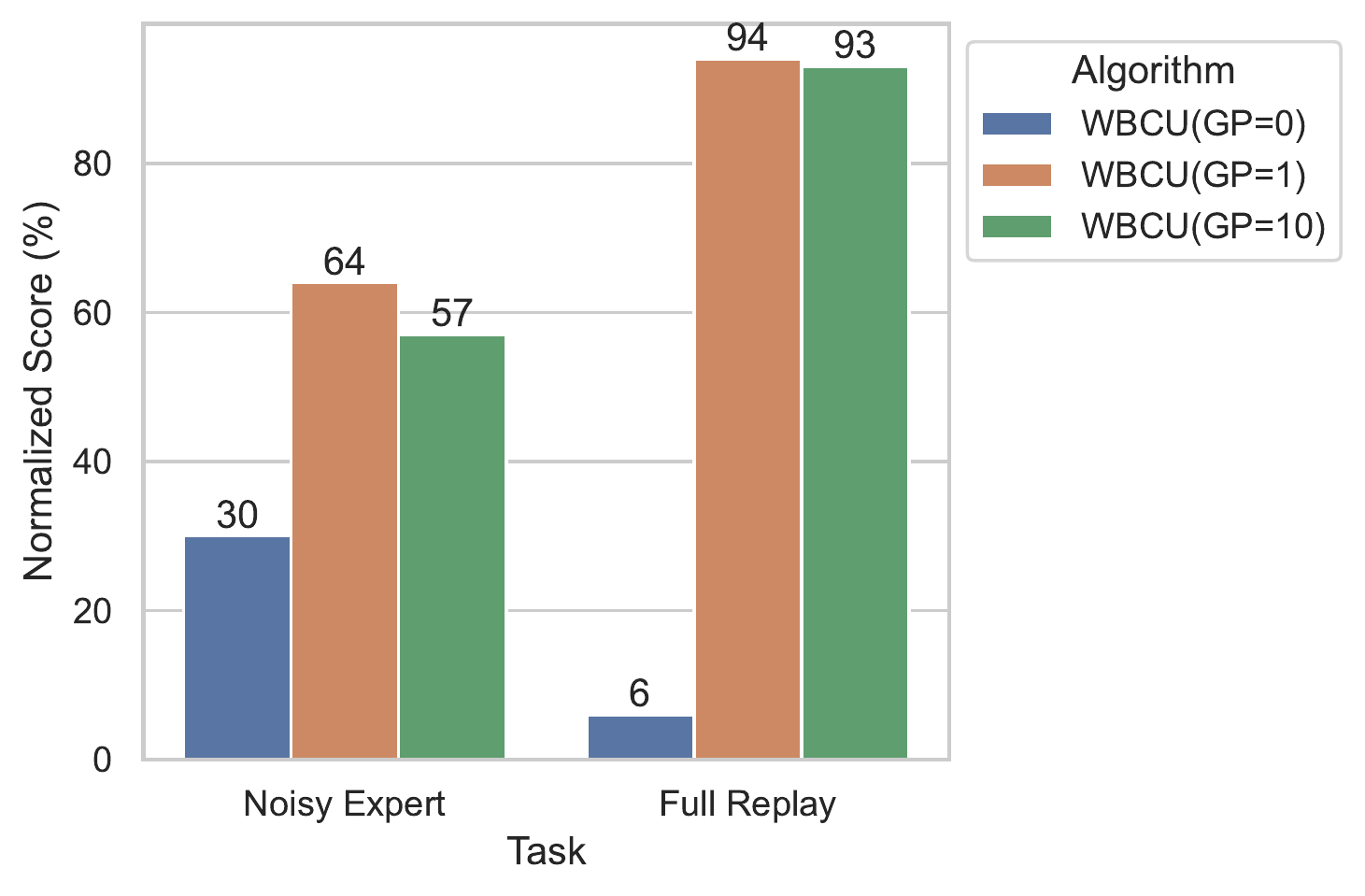}
    \caption{Averaged normalized scores of trained policies of WBCU with gradient penalty (GP). Numbers in the legend indicate the scale of the GP regularization.}
    \label{fig:wbcu_gradient_penalty}
\end{figure}

\textbf{Summary.} Our theoretical analysis (\cref{prop:wbcu_negative} and \cref{thm:wbcu}) discloses that the \dquote{smooth} function approximation is inevitable to achieve satisfying performance. For practitioners, our results would suggest that regularization techniques that control the smooth property of the function approximators may improve the performance, which we empirically verify below.

\textbf{Additional Experiments.}  We note that quite often, non-linear neural networks rather than linear functions are used in practice. For neural networks, the gradient penalty (GP) regularization\footnote{This technique adds a squared loss of the gradient norm to the original loss function; see \cref{appendix:experiment_details} for details.} is known to control the Lipschitz continuous property of the discriminator \citep{gulrajani2017improved}. In particular, a large gradient penalty loss can push the discriminator to prefer 1-Lipschitz continuous functions that are \dquote{smooth}. With the same set-up with experiments in \cref{fig:summary}, we show that the gradient penalty is crucial for the practical performance of WBCU; see \cref{fig:wbcu_gradient_penalty}. A similar phenomenon is also observed for the related algorithm DemoDICE; see \cref{fig:demo_dice_gp} in Appendix.

\section{Conclusion}

We theoretically explore imitation learning with a supplementary dataset, and empirical results corroborate our findings. While our results have several desirable features, they also have shortcomings. One limitation is that we consider the tabular representation in policy learning. However, our theoretical implications may remain unchanged if function approximation is used. Please see \cref{sec:bc_gfa} for discussion, which deserves further investigation. 

Exploring more applications of NBCU and WBCU is an interesting future work. For example, for large language models \citep{radford2018improving}, we may have massive supplementary samples from the Web, while examples with human annotations are limited. Compared with the existing reinforcement-learning-based methods (see, e.g., \citep{stiennon2020learning}), training good policies may be easier and more efficient by the developed imitation learning approaches.

\section*{Acknowledgements}

Ziniu Li would like to thank Yushun Zhang, Yingru Li, Jiancong Xiao, and Dmitry Rybin for reading the manuscript and providing valuable comments. Tian Xu would like to thank Fanming Luo, Zhilong Zhang, and Jingcheng Pang for reading the manuscript and providing helpful comments. Ziniu Li appreciates the helpful discussion with Congliang Chen about a technical lemma.

\bibliographystyle{abbrvnat}
\bibliography{reference.bib}

\newpage
\appendix
\onecolumn
\section{Proof of Results in Section \ref{sec:analysis_of_behavioral_cloning}}

\begin{lem}
\label{lem:binomial_distribution}
For any $N \in \mathbb{N}_{+}$ and $p \in (0, 1)$, if the random variable $X$ follows the binomial distribution, i.e., $X \sim \bin (N, p)$, then we have that
\begin{align*}
    \expect \ls \frac{1}{X+1} \rs \leq \frac{1}{N p}.
\end{align*}
\end{lem}

\begin{proof}
\begin{align*}
    \expect \ls \frac{1}{X+1} \rs &= \sum_{x=0}^N \lp \frac{1}{x+1} \rp  \frac{N!}{x! (N-x)!} p^{x} (1-p)^{N-x}
    \\
    &= \frac{1}{(N+1) p} \sum_{x=1}^{N+1} \lp \frac{(N+1)!}{x! (N+1 - x)!} \rp p^x (1-p)^{N+1-x}
    \\
    &=  \frac{1}{(N+1) p} \lp 1 - (1-p)^{N+1} \rp \leq \frac{1}{N p}. 
\end{align*}
\end{proof}

\subsection{Proof of Theorem \ref{thm:bc_expert}}

When $|\gDE| \geq 1$, by \citep[Theorem 4.2]{rajaraman2020fundamental}, we have that
\begin{align*}
    V(\piE) - \expect_{\gDE} \ls  V(\pi^{\bc}) \rs \leq \frac{4 |\gS| H^2}{9 |\gDE|}.
\end{align*}
When $|\gDE| = 0$, we simply have that
\begin{align*}
    V(\piE) - \expect_{\gDE} \ls  V(\pi^{\bc}) \rs \leq H.
\end{align*}
Therefore, we have the following unified bound.
\begin{align*}
    V(\piE) - \expect_{\gDE} \ls  V(\pi^{\bc}) \rs \leq \frac{|\gS| H^2}{ \max\{ |\gDE|, 1 \}} \leq \frac{2 |\gS| H^2}{ |\gDE|+1}.
\end{align*}
The last inequality follows that $\max\{ x, 1 \} \geq (x + 1)/2, \; \forall x \geq 0$. Finally, notice that $|\gDE| \sim \bin (N_{\tot}, \eta)$. By \cref{lem:binomial_distribution}, we have that
\begin{align*}
     V(\piE) - \expect \ls  V(\pi^{\bc}) \rs \leq \expect \ls \frac{2 |\gS| H^2}{ |\gDE|+1} \rs \leq \frac{2 |\gS| H^2}{ N_{\tot} \eta} = \frac{2 |\gS| H^2}{ N_{\expert} } , 
\end{align*}
which completes the proof.

\subsection{Proof of Theorem \ref{thm:bc_union}}
For analysis, we first define the mixture state-action distribution as follows.
\begin{align*}
     d^{\mix}_h (s, a) \triangleq \eta d^{\piE}_h (s, a) + (1-\eta) d^{\pi^{\beta}}_h (s, a), d^{\mix}_h (s) \triangleq \sum_{a \in \gA} d^{\mix}_h (s, a),  \; \forall (s, a) \in \gS \times \gA, \; \forall h \in [H].   
\end{align*}
Note that in the population level, the marginal state-action distribution of union dataset $\gDU$ in time step $h$ is exactly $d^{\mix}_h$. That is, $d^{\oU}_h (s, a) = d^{\mix}_h (s, a), \; \forall (s, a, h) \in \gS \times \gA \times [H]$. Then we define the mixture policy $\pimix$ induced by $d^{\mix}$ as follows.
\begin{align*}
     \pimix_h (a|s) =
    \begin{cases}
    \frac{d^{\mix}_h (s, a)}{d^{\mix}_h (s) } & \text{if } d^{\mix}_h (s) > 0,
    \\
     \frac{1}{\vert \gA \vert} & \text{otherwise}.
    \end{cases} \; \forall (s, a) \in \gS \times \gA\;, \forall h \in [H].
\end{align*}
From the theory of Markov Decision Processes, we know that (see, e.g., \citep{puterman2014markov}) 
\begin{align*}
    \forall h \in [H], \forall (s, a) \in \gS \times \gA, d^{\pimix}_h (s, a) = d^{\mix}_h (s, a). 
\end{align*}
Therefore, we can obtain that the marginal state-action distribution of union dataset $\gDU$ in time step $h$ is exactly $d^{\pimix}_h$. Then we have the following decomposition.
\begin{align*}
    \expect \ls V(\piE) - V(\pi^{\nbcu}) \rs &= \expect \ls V(\piE) - V (\pimix) +  V (\pimix) -  V(\pi^{\nbcu}) \rs
    \\
    &= \expect \ls V(\piE) - V (\pimix) \rs + \expect \ls  V (\pimix) -  V(\pi^{\nbcu}) \rs
    \\
    &= V(\piE) - V (\pimix) + \expect \ls  V (\pimix) -  V(\pi^{\nbcu}) \rs.  
\end{align*}
For $V(\piE) - V (\pimix)$, we have that 
\begin{align*}
    V(\piE) - V (\pimix) &= \sum_{h=1}^H \sum_{(s, a) \in \gS \times \gA} \lp d^{\piE}_h (s, a) - d^{\pimix}_h (s, a) \rp r_h (s, a)
    \\
    &= \sum_{h=1}^H \sum_{(s, a) \in \gS \times \gA} \lp d^{\piE}_h (s, a) - d^{\mix}_h (s, a) \rp r_h (s, a) 
    \\
    &= (1-\eta) \sum_{h=1}^H \sum_{(s, a) \in \gS \times \gA} \lp d^{\piE}_h (s, a) - d^{\pi^{\beta}}_h (s, a) \rp r_h (s, a)
    \\
    &= (1-\eta) \lp V (\piE) - V (\pi^{\beta}) \rp.
\end{align*}
The last equation follows the dual formulation of policy value \citep{puterman2014markov}. Besides, notice that $\expect \ls  V (\pimix) -  V(\pi^{\nbcu}) \rs$ is exactly the imitation gap of BC when regarding $\pimix$ and $\gDU$ as the expert policy and expert demonstrations, respectively. By \citep[Theorem 4.4]{rajaraman2020fundamental}, we have that
\begin{align*}
    \expect \ls  V (\pimix) -  V(\pi^{\nbcu})  \rs \lesssim  \frac{|\gS| H^2 \log (N_{\tot})}{N_{\tot}}.
\end{align*}
Combining the above two equations yields that
\begin{align*}
    \expect \ls V(\piE) - V(\pi^{\nbcu}) \rs \lesssim (1-\eta) \lp V (\piE) - V (\pi^{\beta}) \rp + \frac{|\gS| H^2 \log (N_{\tot})}{N_{\tot}}. 
\end{align*}

\subsection{Proof of Proposition \ref{prop:bc_union_lower_bound}}

We consider the instance named Standard Imitation in \citep{xu2021error}; see \cref{fig:standard_imitation}. In Standard Imitation, each state is an absorbing state. In each state, only by taking the action $a^1$, the agent can obtain the reward of $1$. Otherwise, the agent obtains zero rewards. The initial state distribution is a uniform distribution, i.e., $\rho (s) = 1/|\gS|, \forall s \in \gS$.

\begin{figure}[htbp]
    \centering
    \includegraphics[width=0.5\linewidth]{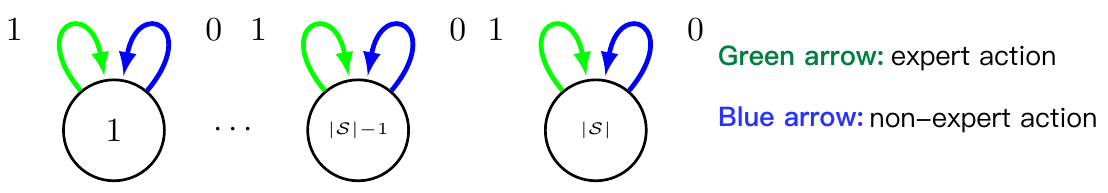}
    \caption{The Standard Imitation MDP in \citep{xu2021error}.}
    \label{fig:standard_imitation}
\end{figure}

We consider that the expert policy $\piE$ always takes the action $a^1$ while the behavioral policy $\pi^{\beta}$ always takes another action $a^2$. Formally, $\piE_h (a^1|s) = 1, \forall s \in \gS, h \in [H], \; \pi^{\beta}_h (a^2|s) = 1, \forall s \in \gS, h \in [H]$. It is direct to calculate that $V (\piE) = H$ and $V (\pi^{\beta}) = 0$. The supplementary dataset $\gDS$ and the expert dataset $\gDE$ are collected according to \cref{asmp:dataset_collection}. The mixture state-action distribution can be calculated as $\forall s \in \gS, h \in [H]$,
\begin{align*}
    & d^{\mix}_h (s, a^1) = \eta d^{\piE}_h (s, a^1) + (1-\eta) d^{\pi^{\beta}}_h (s, a^1) = \eta d^{\piE}_h (s, a^1) = \eta \rho (s),
    \\
    & d^{\mix}_h (s, a^2) = \eta d^{\piE}_h (s, a^2) + (1-\eta) d^{\pi^{\beta}}_h (s, a^2) = (1-\eta) d^{\pi^{\beta}}_h (s, a^2) = (1-\eta) \rho (s).  
\end{align*}
Note that in the population level, the marginal distribution of the union dataset $\gDU$ in time step $h$ is exactly $d^{\mix}_h$. The mixture policy induced by $d^{\mix}$ can be formulated as
\begin{align*}
\pimix_h (a^1|s) = \eta, \pimix_h (a^2|s) = 1-\eta, \forall s \in \gS, h \in [H]. 
\end{align*}
Just like before, we have $d^{\pimix}_h (s, a) = d^{\mix}_h (s, a)$. The policy value of $\pimix$ can be calculated as 
\begin{align*}
    V (\pimix) = \sum_{h=1}^H \sum_{(s, a) \in \gS \times \gA} d^{\mix}_h (s, a) r_h (s, a) = \sum_{h=1}^H \sum_{s \in \gS} d^{\mix}_h (s, a^1) = \eta H.  
\end{align*}
The policy $\pi^{\nbcu}$ can be formulated as
\begin{align}
 \forall h \in [H], \quad    \pi^{\nbcu}_h(a|s) = \left\{ \begin{array}{ll}
      \frac{n^{\oU}_h(s, a)}{\sum_{a^\prime} n^{\oU}_h(s, a^\prime)}  & \text{if}\sum_{a^\prime} n^{\oU}_h(s, a^\prime) > 0  \\
        \frac{1}{|\gA|} & \text{otherwise}  
    \end{array} \right.
\end{align}
We can view that the BC's policy learned on the union dataset mimics the mixture policy $\pimix$. In the following part, we analyze the lower bound on the imitation gap of $\pi^{\nbcu}$.
\begin{align*}
    \expect \ls V (\piE) - V (\pibcunion) \rs &= V (\piE) - V (\pimix) + \expect \ls V (\pimix) - V (\pibcunion)  \rs
    \\
    &= H - \eta H + \expect \ls V (\pimix) - V (\pibcunion)  \rs
    \\
    &= (1-\eta) (V (\piE) - V (\pi^\beta)) + \expect \ls V (\pimix) - V (\pibcunion)  \rs. 
\end{align*}
Then we consider the term $\expect \ls V (\pimix) - V (\pibcunion)  \rs$.
\begin{align*}
     V (\pimix) - V (\pibcunion) &= \sum_{h=1}^H \sum_{(s, a) \in \gS \times \gA} \lp d^{\pimix}_h (s, a) - d^{\pibcunion}_h (s, a) \rp r_h (s, a)
    \\
    &= \sum_{h=1}^H \sum_{(s, a) \in \gS \times \gA} \rho (s) \lp \pimix_h (a|s) - \pibcunion_h (a|s) \rp r_h (s, a)
    \\
    &= \sum_{h=1}^H \sum_{(s, a) \in \gS \times \gA} \rho (s) \lp \pimix_h (a|s) - \pibcunion_h (a|s) \rp r_h (s, a) \indict \{ n^{\oU}_h (s) > 0  \}
    \\
    &\quad + \sum_{h=1}^H \sum_{(s, a) \in \gS \times \gA} \rho (s) \lp \pimix_h (a|s) - \pibcunion_h (a|s) \rp r_h (s, a) \indict \{ n^{\oU}_h (s) = 0  \}. 
\end{align*}
We take expectation over the randomness in $\gDU$ on both sides and obtain that
\begin{align}
\expect \ls V (\pimix) - V (\pibcunion) \rs  &= \expect \ls \sum_{h=1}^H \sum_{(s, a) \in \gS \times \gA} \rho (s) \lp \pimix_h (a|s) - \pibcunion_h (a|s) \rp r_h (s, a) \indict \{ n^{\oU}_h (s) > 0  \} \rs \nonumber
    \\
    &\quad + \expect \ls \sum_{h=1}^H \sum_{(s, a) \in \gS \times \gA} \rho (s) \lp \pimix_h (a|s) - \pibcunion_h (a|s) \rp r_h (s, a) \indict \{ n^{\oU}_h (s) = 0  \} \rs. \label{eq:lower_bound_decomposition} 
\end{align}
For the first term in RHS, we have that
\begin{align*}
    &\quad \expect \ls \sum_{h=1}^H \sum_{(s, a) \in \gS \times \gA} \rho (s) \lp \pimix_h (a|s) - \pibcunion_h (a|s) \rp r_h (s, a) \indict \{ n^{\oU}_h (s) > 0  \} \rs
    \\
    &= \sum_{h=1}^H \sum_{(s, a) \in \gS \times \gA} \rho (s) r_h (s, a) \expect \ls  \lp \pimix_h (a|s) - \pibcunion_h (a|s) \rp \indict \{ n^{\oU}_h (s) > 0  \} \rs
    \\
    &= \sum_{h=1}^H \sum_{(s, a) \in \gS \times \gA} \rho (s) r_h (s, a) \sP \lp n^{\oU}_h (s) > 0 \rp \expect \ls  \lp \pimix_h (a|s) - \pibcunion_h (a|s) \rp \bigg|  n^{\oU}_h (s) > 0  \rs
    \\
    &= 0.
\end{align*}
The last equation follows the fact that when $n^{\oU}_h (s) > 0$, $\pibcunion_h (a|s)$ is an unbiased maximum likelihood estimation of $\pimix_h (a|s)$. For the remaining term, we have that
\begin{align*}
    &\quad \expect \ls \sum_{h=1}^H \sum_{(s, a) \in \gS \times \gA} \rho (s) \lp \pimix_h (a|s) - \pibcunion_h (a|s) \rp r_h (s, a) \indict \{ n^{\oU}_h (s) = 0  \} \rs
    \\
    &= \sum_{h=1}^H \sum_{(s, a) \in \gS \times \gA} \rho (s) r_h (s, a) \expect \ls   \lp \pimix_h (a|s) - \pibcunion_h (a|s) \rp  \indict \{ n^{\oU}_h (s) = 0  \} \rs
    \\
    &= \sum_{h=1}^H \sum_{(s, a) \in \gS \times \gA} \rho (s) r_h (s, a) \sP \lp n^{\oU}_h (s) = 0 \rp \expect \ls   \lp \pimix_h (a|s) - \pibcunion_h (a|s) \rp \bigg| n^{\oU}_h (s) = 0 \rs
    \\
    &= \sum_{h=1}^H \sum_{(s, a) \in \gS \times \gA} \rho (s) r_h (s, a) \sP \lp n^{\oU}_h (s) = 0 \rp \lp \pi^{\mix}_h (a|s) - \frac{1}{\vert \gA \vert} \rp
    \\
    &\overset{(a)}{=} \sum_{h=1}^H \sum_{s \in \gS} \rho (s) \sP \lp n^{\oU}_h (s) = 0 \rp \lp \eta - \frac{1}{\vert \gA \vert} \rp
    \\
    & \overset{(b)}{=} H \lp \eta - \frac{1}{\vert \gA \vert} \rp \sum_{s \in \gS} \rho (s) \sP \lp n^{\oU}_1 (s) = 0 \rp. 
\end{align*}
In the equation $(a)$, we use the fact that $r_h (s, a^1) = 1, r_h (s, a) = 0, \forall a \in \gA \setminus \{a^1 \}$. In the equation $(b)$, since each state is an absorbing state, we have that $\sP \lp n^{\oU}_h (s) = 0 \rp = \sP \lp n^{\oU}_1 (s) = 0 \rp, \forall h \in [H]$.

We consider two cases. In the first case of $\eta \geq 1/\vert \gA \vert$, we directly have that
\begin{align*}
    \expect \ls \sum_{h=1}^H \sum_{(s, a) \in \gS \times \gA} \rho (s) \lp \pimix_h (a|s) - \pibcunion_h (a|s) \rp r_h (s, a) \indict \{ n^{\oU}_h (s) = 0  \} \rs \geq 0.
\end{align*}
By \cref{eq:lower_bound_decomposition}, we have that
\begin{align*}
    \expect \ls V (\pimix) - V (\pibcunion) \rs \geq 0,
\end{align*}
which implies that
\begin{align*}
    \expect \ls V (\piE) - V (\pibcunion) \rs \geq  (1-\eta) (V (\piE) - V (\pi^\beta)).
\end{align*}
In the second case of $\eta < 1/\vert \gA \vert$, we have that
\begin{align*}
     H \lp \eta - \frac{1}{\vert \gA \vert} \rp \sum_{s \in \gS} \rho (s) \sP \lp n^{\oU}_1 (s) = 0 \rp
    &\overset{(a)}{\geq} - \lp \frac{1}{\vert \gA \vert} - \eta \rp H \exp \lp - \frac{N_{\tot}}{|\gS|} \rp
    \\
    &\geq - (1-\eta) H \exp \lp - \frac{N_{\tot}}{|\gS|} \rp
    \\
    &\overset{(b)}{\geq} -\frac{ (1-\eta) H}{2}.
\end{align*}
In the inequality $(a)$, we use that 
\begin{align*}
    \sum_{s \in \gS} \rho (s) \sP \lp n^{\oU}_1 (s) = 0 \rp = \sum_{s \in \gS} \rho (s) (1-\rho (s))^{N_{\tot}} = \lp 1-\frac{1}{\vert \gS \vert} \rp^{N_{\tot}} \leq \exp \lp - \frac{N_{\tot}}{\vert \gS \vert} \rp.
\end{align*}
The inequality $(b)$ holds since we consider the range where $N_{\tot} \geq \vert \gS \vert \log (2) $. By \cref{eq:lower_bound_decomposition}, we have that
\begin{align*}
    \expect \ls V (\pimix) - V (\pibcunion) \rs \geq - \frac{ (1-\eta) H}{2}.
\end{align*}
This implies that
\begin{align*}
    \expect \ls V (\piE) - V (\pibcunion) \rs &\geq  (1-\eta) (V (\piE) - V (\pi^\beta)) - \frac{(1-\eta) H}{2}
    \\
    &=  \frac{(1-\eta)}{2} (V (\piE) - V (\pi^\beta)). 
\end{align*}
In both cases, we prove that $\expect \ls V (\piE) - V (\pibcunion) \rs \gtrsim (1-\eta) (V (\piE) - V (\pi^\beta)) $ and thus complete the proof.

\subsection{Proof of Theorem \ref{thm:bc_imitation_gap_expert}}
When $\pi^\beta = \piE$, $\expect \ls V (\piE) - V (\pibcunion) \rs$ is exactly the imitation gap of BC on $\gDU$ with expert policy $\piE$.  By \citep[Theorem 4.2]{rajaraman2020fundamental}, we finish the proof.

\subsection{Proof of Theorem \ref{thm:bc_imitation_gap_no_overlap}}

According to the policy difference lemma in \citep{kakade2002pg}, we have that
\begin{align*}
    V (\piE) - V (\pibcunion) &= \sum_{h=1}^H \expect_{s_h \sim d^{\piE}_h (\cdot)} \ls Q^{\pibcunion}_h (s_h, \piE_h (s_h)) - \sum_{a \in \gA} \pibcunion_h (\cdot|s_h) Q^{\pibcunion}_h (s_h, a)  \rs
    \\
    &= \sum_{h=1}^H \expect_{s_h \sim d^{\piE}_h (\cdot)} \ls Q^{\pibcunion}_h (s_h, \piE_h (s_h)) (1-\pibcunion_h (\piE_h (s_h)|s_h))  \rs  
    \\
    &\leq H \sum_{h=1}^H \expect_{s \sim d^{\piE}_h (\cdot)} \ls \expect_{a \sim \pibcunion_h (\cdot|s)} \ls \indict \lb a \not= \piE_h (s) \rb \rs \rs,
\end{align*}
where $Q^{\pi}_h(s, a) = \expect[\sum_{t=h}^{H} r(s_t, a_t) |(s_h, a_h) = (s, a)]$ is the state-action value function for a policy $\pi$ and $\piE_h(s_h)$ denotes the expert action. By assumption, we have $\forall h \in [H]$, $\supp (d^{\piE}_h (\cdot)) \cap \supp (d^{\pi^{\beta}}_h (\cdot)) = \emptyset$. Therefore, we know that the union state-action pairs in $\gDU$ does not affect $\pibcunion$ on expert states. As a result, $\pibcunion$ exactly takes expert actions on states visited in $\gDE$. Then we have that
\begin{align*}
    V (\piE) - V (\pibcunion) \leq H \sum_{h=1}^H \expect_{s \sim d^{\piE}_h (\cdot)} \ls \indict \lb s \notin \gS_h (\gDE) \rb \rs,
\end{align*}
where $\gS_h (\gDE)$ is the set of states in time step $h$ in $\gDE$. Taking expectation over the randomness within $\gDE$ on both sides yields that
\begin{align*}
    V (\piE) - \expect_{\gDE} \ls V (\pibcunion) \rs \leq H \expect_{\gDE} \ls \sum_{h=1}^H \expect_{s \sim d^{\piE}_h (\cdot)} \ls \indict \lb s \notin \gS_h (\gDE) \rb \rs \rs.
\end{align*}
By \citep[Lemma A.1]{rajaraman2020fundamental}, we further derive that
\begin{align*}
    V (\piE) - \expect_{\gDE} \ls V (\pibcunion) \rs \leq  \frac{|\gS| H^2}{ \max \{ |\gDE|, 1\}} \leq  \frac{2 |\gS| H^2}{ |\gDE|+ 1}. 
\end{align*}
We take expectation over the binomial variable $|\gDE|$ and have that
\begin{align*}
    V (\piE) - \expect \ls V (\pibcunion) \rs \leq \expect \ls \frac{2 |\gS| H^2}{ |\gDE|+ 1} \rs \leq \frac{2 |\gS| H^2}{N_{\tot} \eta} = \frac{2 |\gS| H^2}{N_{\expert}} , 
\end{align*}
which completes the proof.

\section{Proof of Results in Section \ref{sec:analysis_of_weighted_behavioral_cloning}}
\label{sec:proof_of_weighted_behavioral_cloning}

\subsection{Proof of Proposition \ref{prop:wbcu_negative}}

In the tabular case, with the first-order optimality condition, we have $c^\star_h (s, a) = \widehat{d^{\expert}_h} (s, a) /( \widehat{d^{\expert}_h} (s, a) + \widehat{d^{\oU}_h} (s, a)) $. By \cref{eq:from_c_to_w}, we have 
\begin{align*}
   \widehat{d^{\oU}_h}(s, a) w_h(s, a) =  \widehat{d^{\oU}_h}(s, a)  \times \frac{\widehat{d^{\expert}_h} (s, a)}{\widehat{d^{\oU}_h} (s, a)} = \widehat{d^{\expert}_h}(s, a). 
\end{align*}
Hence, the learning objective \eqref{eq:weighted_bc} reduces to \eqref{eq:bc_opt}.

\subsection{Proof of Lemma \ref{lem:lipschitz_continuity}}
Recall that
\begin{align*}
    \Delta_h(\theta)= \min_{(s, a) \in \gDE_h \cup \gD^{\oS, 1}_h}   \langle  \theta,  \phi_h (s, a) \rangle - \max_{(s^\prime, a^\prime) \in \gD^{\oS, 2}_h}  \langle  \theta,  \phi_h (s^\prime, a^\prime) \rangle.
\end{align*}
Then we have that
\begin{align*}
 \Delta_h(\bar{\theta}_h)  - \Delta_h(\theta) &= \min_{(s, a) \in \gDE_h \cup \gD^{\oS, 1}_h}   \langle  \bar{\theta}_h,  \phi_h (s, a) \rangle - \max_{(s^\prime, a^\prime) \in \gD^{\oS, 2}_h}  \langle  \bar{\theta}_h,  \phi_h (s^\prime, a^\prime) \rangle
 \\
 &\quad - \min_{(s, a) \in \gDE_h \cup \gD^{\oS, 1}_h}   \langle  \theta,  \phi_h (s, a) \rangle + \max_{(s^\prime, a^\prime) \in \gD^{\oS, 2}_h}  \langle  \theta,  \phi_h (s^\prime, a^\prime) \rangle
 \\
 &\overset{(a)}{\leq} \langle  \bar{\theta}_h,  \phi_h (s^1, a^1) \rangle - \langle  \bar{\theta}_h,  \phi_h (s^2, a^2) \rangle -   \langle  \theta,  \phi_h (s^1, a^1) \rangle +  \langle  \theta,  \phi_h (s^2, a^2) \rangle
 \\
 &= \langle \bar{\theta}_h - \theta, \phi_h (s^1, a^1) - \phi_h (s^2, a^2)   \rangle
 \\
 &\overset{(b)}{\leq} \lnorm \bar{\theta}_h - \theta \rnorm \lnorm \phi_h (s^1, a^1) - \phi_h (s^2, a^2) \rnorm.
\end{align*}
In inequality $(a)$, we utilize the facts that $(s^1, a^1) \in \argmin_{ (s, a) \in \gDE_h \cup \gD^{\oS, 1}_h} \langle \theta_h,  \phi_h (s, a) \rangle$ and \\ $(s^2, a^2) \in \argmax_{(s, a) \in  \gD^{\oS, 2}_h } \langle \theta_h,  \phi_h (s, a) \rangle$. Inequality $(b)$ follows the Cauchy–Schwarz inequality. Let $L_h = \lnorm \phi_h (s^1, a^1) - \phi_h (s^2, a^2) \rnorm$ and we finish the proof.

\subsection{Proof of Lemma \ref{lem:strong_convexity}}
First, by Taylor's Theorem, there exists $\theta_h^\prime \in \{ \theta \in \reals^d: \theta^t = \theta^\star_h + t (\widebar{\theta}_h - \theta^\star_h), \; \forall t \in [0, 1] \}$ such that
\begin{align}
    \gL_h (\widebar{\theta}_h) &= \gL_h (\theta^\star_h) + \langle \nabla \gL_h (\theta^\star_h), \widebar{\theta}_h - \theta^\star_h   \rangle + \frac{1}{2} \lp \widebar{\theta}_h - \theta^\star_h \rp^\top \nabla^2 \gL_h (\theta_h^\prime) \lp \widebar{\theta}_h - \theta^\star_h \rp \nonumber
    \\
    &= \gL_h (\theta^\star_h) + \frac{1}{2} \lp \widebar{\theta}_h - \theta^\star_h \rp^\top \nabla^2 \gL_h (\theta_h^\prime) \lp \widebar{\theta}_h - \theta^\star_h \rp. \label{eq:taylor_approximation}
\end{align}
The last equality follows the optimality condition that $\nabla \gL_h (\theta^\star_h) = 0$. Then, our strategy is to prove that the smallest eigenvalue of the Hessian matrix $\nabla^2 \gL_h (\theta_h^\prime)$ is positive, i.e., $\lambda_{\min} (\nabla^2 \gL_h (\theta_h^\prime)) > 0$. We first calculate the Hessian matrix $\nabla^2 \gL_h (\theta_h^\prime)$. Given $\gDE$ and $\gDU$, we define the function $G: \reals^{(|\gDE| + |\gDU|)} \rightarrow \reals$ as
\begin{align*}
    G (v) \triangleq \frac{1}{|\gDE|} \sum_{i=1}^{|\gDE|} g (v_i) + \frac{1}{|\gDU|} \sum_{j=1}^{|\gDU|} g (v_j),  
\end{align*}
where $v_i$ is the $i$-th element in the vector $v \in \reals^{(|\gDE| + |\gDU|)}$ and $g (x) = \log \lp 1 + \exp (x) \rp $ is a real-valued function. Besides, we use $B_h \in \reals^{ (|\gDE| + |\gDU|) \times d}$ to denote the matrix whose $i$-th row $B_{h, i} = - y_i \phi_h (s^i, a^i)^\top$, and $y_i = 1$ if $ (s^i, a^i) \in \gDE_h$, $y_i = -1$ if $ (s^i, a^i) \notin \gDE_h$. Then the objective function can be reformulated as
\begin{align*}
\gL_h (\theta_h) &= \sum_{(s, a)} \widehat{d^{\expert}_h}(s, a) \ls \log \lp 1 + \exp \lp - \langle \phi_h (s,a), \theta_h \rangle \rp \rp \rs + \sum_{(s, a)} \widehat{d^{\oU}_h} (s, a) \ls \log \lp 1 + \exp \lp \langle \phi_h (s,a), \theta_h \rangle \rp \rp \rs
\\
&= \frac{1}{|\gDE|} \sum_{(s, a) \in \gDE} \log \lp 1 + \exp \lp - \langle \phi_h (s,a), \theta_h \rangle \rp \rp + \frac{1}{|\gDU|} \sum_{(s, a) \in \gDU} \log \lp 1 + \exp \lp \langle \phi_h (s,a), \theta_h \rangle \rp \rp  
\\
&= G (B_h \theta_h).
\end{align*}
Then we have that $\nabla^2 \gL_h (\theta_h) = B_h^\top \nabla^2 G (B_h \theta_h) B_h$, where
\begin{align*}
     \nabla^2 G (B_h \theta_h) = \diag \lp \frac{g^{\prime \prime} ((B_h \theta_h)_1)}{|\gDE|}, \ldots, \frac{g^{\prime \prime} ((B_h \theta_h)_{|\gDE|})}{|\gDE|}, \frac{g^{\prime \prime} ((B_h \theta_h)_{|\gDE| + 1})}{|\gDE |+ |\gDU|}, \ldots, \frac{g^{\prime \prime} ((B_h \theta_h)_{|\gDE |+ |\gDU|})}{|\gDE |+ |\gDU|}    \rp.
\end{align*}
Here $g^{\prime \prime} (x) = \sigma (x) (1-\sigma (x))$, where $\sigma (x) = 1/(1+\exp (-x))$ is the sigmoid function. The eigenvalues of $\nabla^2 G (B_h \theta_h)$ are
\begin{align*}
    \lb \frac{g^{\prime \prime} ((B_h \theta_h)_1)}{|\gDE|}, \ldots, \frac{g^{\prime \prime} ((B_h \theta_h)_{|\gDE|})}{|\gDE|}, \frac{g^{\prime \prime} ((B_h \theta_h)_{|\gDE| + 1})}{|\gDE |+ |\gDU|}, \ldots, \frac{g^{\prime \prime} ((B_h \theta_h)_{|\gDE |+ |\gDU|})}{|\gDE |+ |\gDU|} \rb.
\end{align*}
Notice that $\theta_h^\prime \in \{ \theta \in \reals^d: \theta^t = \theta^\star_h + t (\widebar{\theta}_h - \theta^\star_h), \; \forall t \in [0, 1] \}$. For a matrix $A$, we use $\lambda_{\min} (A)$ to denote the minimal eigenvalue of $A$. Here we claim that the minimum of the minimal eigenvalues of $\nabla^2 G (B_h \theta^t)$ over $t \in [0, 1]$ is achieved at $t=0$ or $t=1$. That is,
\begin{align*}
    \min \{ \lambda_{\min} (\nabla^2 G (B_h \theta^t)) : \forall t \in [0, 1] \} = \min \{\lambda_{\min} ( \nabla^2 G (B_h \theta^0)), \lambda_{\min} ( \nabla^2 G (B_h \theta^1))  \}.
\end{align*}
We prove this claim as follows. For any $t \in [0, 1]$, we use $\{ \lambda_1 (t), \ldots,  \lambda_{|\gDE| + |\gDU|} (t)\}$ to denote the eigenvalues of $\nabla^2 G (B_h \theta^t)$. For each $i \in [|\gDE| + |\gDU|]$, we consider $\lambda_i (t) :[0, 1] \rightarrow \reals$ as a function of $t$. Specifically,
\begin{align*}
    \lambda_{i} (t) = \begin{cases}
    \frac{g^{\prime \prime} ((B_h \theta^\star_h)_{i} + t (B_h (\widebar{\theta}_h - \theta^\star_h))_{i})}{|\gDE|}, &\quad \text{if } i \in [|\gDE|]
    \\
    \frac{g^{\prime \prime} ((B_h \theta^\star_h)_{i} + t (B_h (\widebar{\theta}_h - \theta^\star_h))_{i})}{|\gDE| + |\gDU|}, &\quad \text{otherwise}.
    \end{cases}
\end{align*}
We observe that $g^{\prime \prime \prime} (x) = \sigma (x) (1-\sigma (x)) (1-2\sigma (x))$ which satisfies that $ \forall x \leq 0, \; g^{\prime \prime \prime} (x) \geq 0, $ and $\forall x \geq 0, \; g^{\prime \prime \prime} (x) \leq 0 $. Therefore, we have that the minimum of $\lambda_i (t)$ over $t \in [0, 1]$ must be achieved at $t=0$ or $t=1$. That is,
\begin{align}
\label{eq:min_eigenvalues_boundary}
    \min_{t \in [0, 1]} \lambda_i (t) = \min \{ \lambda_i (0), \lambda_i (1) \}.
\end{align}

For any $t \in [0, 1]$, we define $i^t \in [|\gDE| + |\gDU|]$ as the index of the minimal eigenvalue of $\nabla^2 G (B_h \theta^t)$, i.e., $\lambda_{i^t} (t) = \lambda_{\min} (\nabla^2 G (B_h \theta^t))$. Then we have that
\begin{align*}
    \min \{ \lambda_{\min} (\nabla^2 G (B_h \theta^t)) : \forall t \in [0, 1] \} &= \min \{ \lambda_{i^t} (t): \forall t \in [0, 1] \}
    \\
    &\overset{(a)}{=} \min \{ \min \{ \lambda_{i^t} (0), \lambda_{i^t} (1 ) \}: \forall t \in [0, 1] \}
    \\
    &= \min \{ \lambda_{i^0} (0), \lambda_{i^1} (1) \}
    \\
    &\overset{(b)}{=} \min \{ \lambda_{\min} (\nabla^2 G (B_h \theta^0)), \lambda_{\min} (\nabla^2 G (B_h \theta^1)) \} 
\end{align*}
Equality $(a)$ follows \eqref{eq:min_eigenvalues_boundary} and equality $(b)$ follows that $\lambda_{i^0} (0)$ and $\lambda_{i^1} (1)$ are the minimal eigenvalues of $\nabla^2 G (B_h \theta^0)$ and $\nabla^2 G (B_h \theta^1)$, respectively.

In summary, we derive that
\begin{align}
\label{eq:minimal_eigenvalue}
    \min \{ \lambda_{\min} (\nabla^2 G (B_h \theta^t)) : \forall t \in [0, 1] \} = \min \{ \lambda_{\min} (\nabla^2 G (B_h \theta^0)), \lambda_{\min} (\nabla^2 G (B_h \theta^1)) \},
\end{align}
which proves the previous claim.

Further, we consider $\lambda_{\min} \lp \nabla^2 \gL_h (\theta_h)  \rp$.
\begin{align*}
    \lambda_{\min} \lp \nabla^2 \gL_h (\theta_h)  \rp &= \inf_{x \in \reals^d: \lnorm x \rnorm = 1} x^\top  \nabla^2 \gL_h (\theta_h) x
    \\
    &= \inf_{x \in \reals^d: \lnorm x \rnorm = 1} \lp B_h x \rp^\top  \nabla^2 G (B_h \theta_h) \lp B_h x \rp
    \\
    &= \inf_{z \in \operatorname{Im} (B_h)} z^\top \nabla^2 G (B_h \theta_h) z
    \\
    &= \lp \inf_{z \in \operatorname{Im} (B_h)} \lnorm z \rnorm \rp^2 \lambda_{\min} (\nabla^2 G (B_h \theta_h))
    \\
    &\geq \lp \inf_{z \in \operatorname{Im} (B_h)} \lnorm z \rnorm \rp^2 \min \{ \lambda_{\min} (\nabla^2 G (B_h \theta^0)), \lambda_{\min} (\nabla^2 G (B_h \theta^1)) \}. 
\end{align*}
Here $\operatorname{Im} (B_h) = \{ z \in \reals^d: z = B_h x, \lnorm x \rnorm = 1 \}$. The last inequality follows \cref{eq:minimal_eigenvalue}.

In the under-parameterization case where $\rank (A_h) = d$, we have that $\rank (B_h) = d$. Thus, $\operatorname{Im} (B_h)$ is a set of vectors with positive norms, i.e., $ \inf_{z \in \operatorname{Im} (B_h)} \lnorm z \rnorm > 0$. Besides, since $g^{\prime \prime} (x) = \sigma (x) (1-\sigma (x)) > 0$, we also have that 
\begin{align*}
    \min \{ \lambda_{\min} (\nabla^2 G (B_h \theta^0)), \lambda_{\min} (\nabla^2 G (B_h \theta^1)) \} > 0.
\end{align*}
In summary, we obtain that
\begin{align*}
    \lambda_{\min} \lp \nabla^2 \gL_h (\theta_h)  \rp \geq \lp \inf_{z \in \operatorname{Im} (B_h)} \lnorm z \rnorm \rp^2 \min \{ \lambda_{\min} (\nabla^2 G (B_h \theta^0)), \lambda_{\min} (\nabla^2 G (B_h \theta^1)) \} > 0. 
\end{align*}
Then, with \cref{eq:taylor_approximation}, there exists $\tau_h = \lp \inf_{z \in \operatorname{Im} (B_h)} \lnorm z \rnorm \rp^2 \min \{ \lambda_{\min} (\nabla^2 G (B_h \theta^0)), \lambda_{\min} (\nabla^2 G (B_h \theta^1)) \} > 0$ such that
\begin{align*}
    \gL_h (\widebar{\theta}_h) \geq \gL_h (\theta^\star_h) + \frac{\tau_h}{2} \lnorm \widebar{\theta}_h - \theta^\star_h \rnorm^2,
\end{align*}
which completes the proof.

\subsection{Proof of Theorem \ref{thm:wbcu}}
First, invoking \cref{lem:lipschitz_continuity} with $\theta = \theta^\star_h$ yields that
\begin{align*}
    \Delta_h(\theta^\star_h)  \geq \Delta_h(\bar{\theta}_h)  - L_h \lnorm \bar{\theta}_h - \theta^\star_h \rnorm.
\end{align*}
Here $L_h = \lnorm \phi_h (s, a) - \phi_h (s^\prime, a^\prime) \rnorm$ with $(s, a) \in \argmin_{ (s, a) \in \gDE_h \cup \gD^{\oS, 1}_h} \langle \theta^\star_h,  \phi_h (s, a) \rangle$ and \\ $(s^\prime, a^\prime) \in \argmax_{(s, a) \in  \gD^{\oS, 2}_h } \langle \theta^\star_h, \phi_h (s, a) \rangle$. Then, by \cref{lem:strong_convexity}, there exists $\tau_h > 0$ such that 
\begin{align*}
    \gL_h (\widebar{\theta}_h) \geq \gL_h (\theta^\star_h) + \frac{\tau_h}{2} \lnorm \widebar{\theta}_h - \theta^\star_h \rnorm^2.
\end{align*}
This directly implies an upper bound of the distance between $\widebar{\theta}_h$ and $\theta^\star_h$. 
\begin{align*}
    \lnorm \widebar{\theta}_h - \theta^\star_h \rnorm \leq \sqrt{ \frac{2 \lp \gL_h(\bar{\theta}_h) - \gL_h(\theta^{\star}_h) \rp}{\tau_h}}.
\end{align*}
When inequality $\eqref{eq:condition}$ holds, we further have that $\lnorm \widebar{\theta}_h - \theta^\star_h \rnorm < \Delta_h (\bar{\theta}_h) / L_h$. Then we get that
\begin{align*}
    \Delta_h(\theta^\star_h)  \geq \Delta_h(\bar{\theta}_h)  - L_h \lnorm \bar{\theta}_h - \theta^\star_h \rnorm >0,
\end{align*}
which completes the proof.

\subsection{An Example Corresponding to Theorem \ref{thm:wbcu}}
\begin{example}
\label{example:condition}
To illustrate \cref{thm:wbcu}, we consider an example in the feature space $\reals^2$. In particular, for time step $h \in [H]$, we have the expert dataset and supplementary dataset as follows.
\begin{align*}
    &\gDE_h = \lb \lp s^{(1)}, a^{(1)} \rp,  \lp s^{(4)}, a^{(4)} \rp \rb, \; \gDS_h = \lb \lp s^{(2)}, a^{(2)} \rp, \lp s^{(3)}, a^{(3)} \rp  \rb, \; \\
    &\gD^{\oS, 1}_h = \lb \lp s^{(2)}, a^{(2)} \rp \rb, \; \gD^{\oS, 2}_h = \lb \lp s^{(3)}, a^{(3)} \rp \rb.
\end{align*}
The corresponding features are
\begin{align*}
    &\phi_h \lp s^{(1)}, a^{(1)} \rp = (0, 1)^\top ,\;  \phi_h \lp s^{(2)}, a^{(2)} \rp = \lp - \frac{1}{2}, 0 \rp^\top, \\
    &\phi_h \lp s^{(3)}, a^{(3)} \rp = \lp 0, -\frac{1}{2} \rp^\top, \; \phi_h \lp s^{(4)}, a^{(4)} \rp = (-1, 0)^\top.
\end{align*}
Notice that the set of expert-style samples is $\gDE_h \cup \gD^{\oS, 1}_h = \{ ( s^{(1)}, a^{(1)} ), ( s^{(2)}, a^{(2)} ),  ( s^{(4)}, a^{(4)} ) \}$ and the set of non-expert-style samples is $\gD^{\oS, 2}_h = \{  ( s^{(3)}, a^{(3)} ) \}$. It is direct to calculate that the ground-truth parameter that achieves the maximum margin among unit vectors is $\widebar{\theta}_h = (-\sqrt{2}/2, \sqrt{2}/2)^\top$ and the maximum margin is $\Delta_h (\widebar{\theta}_h) = \sqrt{2}/2$. According to \cref{eq:parameterized_discriminator_opt}, for $\theta_h = (\theta_{h, 1}, \theta_{h, 2})^\top$, the optimization objective is 
\begin{align*}
    \gL_h (\theta_h) &= \sum_{(s, a)} \widehat{d^{\expert}_h}(s, a) \ls \log \lp 1 + \exp \lp - \langle \phi_h (s,a), \theta_h \rangle \rp \rp \rs + \sum_{(s, a)} \widehat{d^{\oU}_h} (s, a) \ls \log \lp 1 + \exp \lp \langle \phi_h (s,a), \theta_h \rangle \rp \rp \rs
    \\
    &=  \frac{1}{2} \lp \log \lp 1 + \exp \lp - \theta_{h, 2} \rp \rp + \log \lp 1 + \exp \lp \theta_{h, 1} \rp \rp  \rp 
    \\
    & \quad + \frac{1}{4} \lp \log \lp 1 + \exp \lp  \theta_{h, 2} \rp \rp + \log \lp 1 + \exp \lp - \frac{1}{2}  \theta_{h, 1} \rp \rp \rp \\
    &\quad + \frac{1}{4} \lp \log \lp 1 + \exp \lp -\frac{1}{2}  \theta_{h, 2} \rp \rp + \log \lp 1 + \exp \lp  - \theta_{h, 1} \rp \rp  \rp. 
\end{align*}
We apply CVXPY \citep{diamond2016cvxpy} to calculate the optimal solution $\theta^\star_h \approx (-0.310, 0.993)^\top$ and the objective values $\gL_h (\theta^\star_h) \approx 1.287, \; \gL_h (\widebar{\theta}_h) \approx 1.309 $. Furthermore, we calculate the Lipschitz coefficient $L_h$ appears in \cref{lem:lipschitz_continuity}.
\begin{align*}
    & (s^{(2)}, a^{(2)})  = \argmin_{ (s, a) \in \gDE_h \cup \gD^{\oS, 1}_h} \langle \theta^\star_h,  \phi_h (s, a) \rangle , \; (s^{(3)}, a^{(3)}) \in \argmax_{(s, a) \in  \gD^{\oS, 2}_h } \langle \theta^\star_h,  \phi_h (s, a) \rangle,
    \\
    & L_h = \lnorm \phi_h (s^{(2)}, a^{(2)}) - \phi_h (s^{(3)}, a^{(3)}) \rnorm = \frac{\sqrt{2}}{2}.  
\end{align*}
Then we calculate the parameter of strong convexity $\tau_h$ appears in \cref{lem:strong_convexity}. Based on the proof of \cref{lem:strong_convexity}, our strategy is to calculate the minimal eigenvalue of the Hessian matrix.

First, for $\theta_h = (\theta_{h, 1}, \theta_{h, 2})^\top$, the gradient of $\gL_h (\theta_h)$ is
\begin{align*}
    \nabla \gL_h (\theta_h) &= -\sum_{(s, a) \in \gS \times \gA} \widehat{d^{\expert}_h}(s, a) \sigma (- \langle \phi_h (s,a), \theta_h \rangle) + \sum_{(s, a) \in \gS \times \gA} \widehat{d^{\oU}_h} (s, a)  \sigma \lp  \langle \phi_h (s,a), \theta_h \rangle \rp
    \\
    &= \lp \frac{1}{2} \sigma (\theta_{h, 1}) - \frac{1}{4} \sigma (-\theta_{h, 1}) - \frac{1}{8} \sigma (- \frac{1}{2} \theta_{h, 1} ), \frac{1}{4} \sigma \lp \theta_{h, 2} \rp - \frac{1}{2} \sigma \lp - \theta_{h, 2} \rp - \frac{1}{8} \sigma (- \frac{1}{2} \theta_{h, 2}) \rp^\top.
\end{align*}
Here $\sigma (x) = 1/(1+\exp (-x))$ for $x \in \reals$ is the sigmoid function. Then the Hessian matrix at $\theta_h$ is
\begin{align*}
    \nabla^2 \gL_h (\theta_h) = \begin{pmatrix}
\frac{3}{4} f (\theta_{h, 1}) + \frac{1}{16} f \lp \frac{1}{2} \theta_{h, 1} \rp & 0 \\
0 & \frac{3}{4} f (\theta_{h, 2}) + \frac{1}{16} f \lp \frac{1}{2} \theta_{h, 2} \rp
\end{pmatrix},
\end{align*}
where $f (x) = \sigma (x) (1-\sigma (x))$ and $f(x) = f(-x)$.
For any $t \in [0, 1]$, the eigenvalues of the Hessian matrix at $\theta^t_h = \widebar{\theta}_h + t (\theta^\star_h - \widebar{\theta}_h)$ are
\begin{align*}
\frac{3}{4} f (\theta^t_{h, 1} ) + \frac{1}{16} f \lp \frac{1}{2} \theta^t_{h, 1} \rp, \; \frac{3}{4} f (\theta^t_{h, 2} ) + \frac{1}{16} f \lp \frac{1}{2} \theta^t_{h, 2} \rp. 
\end{align*}
Now, we calculate the minimal eigenvalues of $\nabla^2 \gL_h (\theta^t_h)$. We consider the function 
\begin{align*}
    g(x) = \frac{3}{4} f (x) + \frac{1}{16} f \lp \frac{1}{2} x \rp, \; \forall x \in [a, b]. 
\end{align*}
The gradient is 
\begin{align*}
    g^\prime (x) = \frac{3}{4} \sigma (x) (1-\sigma (x)) (1-2\sigma (x)) + \frac{1}{32} \sigma \lp \frac{1}{2}x \rp \lp 1-\sigma \lp \frac{1}{2}x \rp \rp \lp 1-2\sigma \lp \frac{1}{2}x \rp \rp.  
\end{align*}
We observe that $ \forall x \leq 0, \; g^\prime (x) \geq 0, $ and $\forall x \geq 0, \; g^\prime (x) \leq 0 $. Thus, we have that the minimum of $g (x)$ must be achieved at $x = a$ or $x = b$. Besides, we have that $g (x) = g (-x)$. With the above arguments, we know that the minimal eigenvalue is $g (0.993) \approx 0.163$ and $\tau_h \approx 0.163$.
Then we can calculate that
\begin{align*}
   \sqrt{ \frac{2 \lp \gL_h(\bar{\theta}_h) - \gL_h(\theta^{\star}_h) \rp}{\tau_h}} \approx 0.520,\; \frac{\Delta_h (\bar{\theta}_h)}{L_h} = 1.
\end{align*}
The inequality \eqref{eq:condition} holds.
\end{example}

\subsection{Additional Analysis of WBCU With Function Approximation}
Here we provide an additional theoretical result for WBCU with function approximation when $d = 1$.
\begin{thm}
\label{thm:wbcu_1d}
Consider $d=1$. For any $h \in [H]$, if the following inequality holds
\begin{align}
\label{eq:condition_1d}
    \frac{1}{|\gDS|} \sum_{(s, a) \in \gDS_h} \langle \phi_h (s, a), \bar{\theta}_h \rangle < \frac{1}{|\gDE|} \sum_{(s^\prime, a^\prime) \in \gDE_h} \langle \phi_h (s^\prime, a^\prime), \bar{\theta}_h \rangle.
\end{align}
then we have that $\Delta_h (\theta^\star_h) > 0$. Furthermore, the inequality \eqref{eq:condition_1d} is also a necessary condition.
\end{thm}
Theorem \ref{thm:wbcu_1d} provides a clean and sharp condition to guarantee that the learned discriminator can perfectly classify the high-quality samples from $\gDE_h$ and $\gD^{\oS, 1}_h$, and low-quality samples from $\gD^{\oS, 2}_h$. In particular, inequality \eqref{eq:condition_1d} means that by the ground truth parameter $\bar{\theta}_h$, the average score of samples in supplementary dataset is lower than that of samples in expert dataset. This condition is easy to satisfy because the supplementary dataset contains bad samples from $\gD^{\oS, 2}_h$ whose scores are much lower. In contrast, $\gDE$ only contains expert-type samples with high scores.

\begin{proof}
Recall the objective function
\begin{align*}
    \min_{\theta_h}  \gL_h (\theta_h) = \sum_{(s, a)} \widehat{d^{\expert}_h}(s, a) \log \lp 1 + \exp \lp - \langle \phi_h (s,a), \theta_h \rangle \rp \rp  + \sum_{(s, a)} \widehat{d^{\oU}_h} (s, a)  \log \lp 1 + \exp \lp \langle \phi_h (s,a), \theta_h \rangle \rp \rp.
\end{align*}
Consider the function $f (x) = \log (1+\exp (x))$. We have that $f^{\prime \prime} (x) = \sigmoid (x) (1-\sigmoid (x)) \geq 0$, where $\sigmoid (x) = 1/(1+\exp (-x))$. As such, $f^{\prime \prime} (x)$ is a convex function. Notice that the operations of composition with affine function and non-negative weighted sum preserve convexity \citep{boyd2004convex}. Therefore, $\gL_h (\theta_h)$ is a convex function with respect to $\theta_h$. Then for any $\widetilde{\theta} \in \reals$, it holds that
\begin{align*}
    \gL_h (\theta_h) \geq \gL_h (\widetilde{\theta}) + \nabla \gL_h (\widetilde{\theta}) \lp \theta_h - \widetilde{\theta}  \rp. 
\end{align*}
Setting $\theta_h = \theta^\star_h$ and $\widetilde{\theta} = 0$ yields that 
\begin{align}
\label{eq:convexity_first_order}
    \nabla \gL_h (0) \theta^\star_h \leq \gL_h (\theta^\star_h) - \gL_h (0). 
\end{align}
The gradient of $\gL_h (\theta_h)$ is of the form
\begin{align*}
    \nabla \gL_h (\theta_h) = -\sum_{(s, a) \in \gS \times \gA} \widehat{d^{\expert}_h}(s, a) \sigma (- \langle \phi_h (s,a), \theta_h \rangle) \phi_h (s, a) + \sum_{(s, a) \in \gS \times \gA} \widehat{d^{\oU}_h} (s, a)  \sigma \lp  \langle \phi_h (s,a), \theta_h \rangle \rp \phi_h (s, a).
\end{align*}
Here $\sigmoid (x) = 1/(1+\exp (-x))$ is the sigmoid function. Then we calculate the gradient at $\widetilde{\theta} = 0$.
\begin{align*}
    \nabla \gL_h (0) &= -\sum_{(s, a) \in \gS \times \gA} \widehat{d^{\expert}_h}(s, a) \sigma (0) \phi_h (s, a) + \sum_{(s, a) \in \gS \times \gA} \widehat{d^{\oU}_h} (s, a)  \sigma \lp 0 \rp \phi_h (s, a)
    \\
    &= \frac{1}{2} \lp \frac{\sum_{(s, a) \in \gDU_h} \phi_h (s, a)}{|\gDU|}  - \frac{\sum_{(s, a) \in \gDE_h} \phi_h (s, a)}{|\gDE|}\rp
    \\
    &= \frac{1}{2 |\gDU|} \lp \sum_{(s, a) \in \gDU_h} \phi_h (s, a) - \frac{|\gDU|}{|\gDE|}  \sum_{(s, a) \in \gDE_h} \phi_h (s, a)   \rp
    \\
    &= \frac{1}{2 |\gDU|} \lp \sum_{(s, a) \in \gDS_h} \phi_h (s, a) + \sum_{(s, a) \in \gDE_h} \phi_h (s, a) - \frac{|\gDU|}{|\gDE|}  \sum_{(s, a) \in \gDE_h} \phi_h (s, a)   \rp
    \\
    &= \frac{1}{2 |\gDU|} \lp \sum_{(s, a) \in \gDS_h} \phi_h (s, a)  - \frac{|\gDS|}{|\gDE|}  \sum_{(s, a) \in \gDE_h} \phi_h (s, a)   \rp
    \\
    &= \frac{|\gDS|}{2 |\gDU|} \lp  \frac{1}{|\gDS|} \sum_{(s, a) \in \gDS_h} \phi_h (s, a)  - \frac{1}{|\gDE|}  \sum_{(s, a) \in \gDE_h} \phi_h (s, a)   \rp. 
\end{align*}
Furthermore, with inequality \eqref{eq:condition_1d}, we know that 
\begin{align*}
    \frac{1}{|\gDS|} \sum_{(s, a) \in \gDS_h} \phi_h (s, a)  - \frac{1}{|\gDE|}  \sum_{(s, a) \in \gDE_h} \phi_h (s, a) \not= 0. 
\end{align*}
Therefore, $\nabla \gL_h (0) \not= 0$. With the first-order optimality condition, we know that $0$ is not the optimal solution and $\gL_h (0) > \gL_h (\theta^\star_h)$. Combined with inequality \eqref{eq:convexity_first_order}, we have that
\begin{align*}
    \nabla \gL_h (0) \theta^\star_h \leq \gL_h (\theta^\star_h) - \gL_h (0) < 0.
\end{align*}
This directly implies that
\begin{align*}
    \lp  \frac{1}{|\gDS|} \sum_{(s, a) \in \gDS_h} \phi_h (s, a)  - \frac{1}{|\gDE|}  \sum_{(s, a) \in \gDE_h} \phi_h (s, a)   \rp \theta^\star_h < 0.
\end{align*}
Besides, inequality \eqref{eq:condition_1d} implies that
\begin{align*}
    \lp  \frac{1}{|\gDS|} \sum_{(s, a) \in \gDS_h} \phi_h (s, a)  - \frac{1}{|\gDE|}  \sum_{(s, a) \in \gDE_h} \phi_h (s, a)   \rp \bar{\theta}_h < 0.
\end{align*}
With the above two inequalities, we can derive that $\theta^\star_h \bar{\theta}_h > 0$. Then we have that
\begin{align*}
\Delta_h(\theta^\star_h) &= \min_{(s, a) \in \gDE_h \cup \gD^{\oS, 1}_h}  \theta^\star_h  \phi_h (s, a) - \max_{(s^\prime, a^\prime) \in \gD^{\oS, 2}_h}  \theta^\star_h  \phi_h (s^\prime, a^\prime)
\\
     &= \frac{\theta^\star_h}{\bar{\theta}_h} \lp \min_{(s, a) \in \gDE_h \cup \gD^{\oS, 1}_h}  \bar{\theta}_h  \phi_h (s, a) - \max_{(s^\prime, a^\prime) \in \gD^{\oS, 2}_h}  \bar{\theta}_h  \phi_h (s^\prime, a^\prime) \rp
     \\
     &=\frac{\theta^\star_h}{\bar{\theta}_h}  \Delta_h(\bar{\theta}_h) 
     \\
     &> 0.
\end{align*}
Thus, the sufficiency of condition \eqref{eq:condition_1d} is proved. Next, we prove the necessity of condition \eqref{eq:condition_1d}. That is, if $\Delta_h (\theta^\star_h) > 0$, then 
\begin{align*}
    \frac{1}{|\gDS|} \sum_{(s, a) \in \gDS_h} \langle \phi_h (s, a), \bar{\theta}_h \rangle < \frac{1}{|\gDE|} \sum_{(s^\prime, a^\prime) \in \gDS_h} \langle \phi_h (s^\prime, a^\prime), \bar{\theta}_h \rangle.
\end{align*}
Then we aim to prove that $\theta^\star_h \bar{\theta}_h > 0$. It is easy to obtain that $\theta^\star_h \not= 0$ and $\bar{\theta}_h \not= 0$ since $\Delta_h (\theta^\star_h) > 0$ and $\Delta_h (\bar{\theta}_h) > 0$. Then we consider two cases where $\bar{\theta}_h >0$ and $\bar{\theta}_h <0$.  
\begin{itemize}
    \item Case I ($\bar{\theta}_h >0$). In this case, we have that
    \begin{align*}
        \Delta_h (\bar{\theta}_h) &= \min_{(s, a) \in \gDE_h \cup \gD^{\oS, 1}_h}  \bar{\theta}_h  \phi_h (s, a) - \max_{(s^\prime, a^\prime) \in \gD^{\oS, 2}_h}  \bar{\theta}_h  \phi_h (s^\prime, a^\prime)
        \\
        &= \bar{\theta}_h \lp \min_{(s, a) \in \gDE_h \cup \gD^{\oS, 1}_h} \phi_h (s, a) - \max_{(s^\prime, a^\prime) \in \gD^{\oS, 2}_h}  \phi_h (s^\prime, a^\prime)  \rp.
    \end{align*}
    Then $\Delta_h (\bar{\theta}_h) > 0$ implies that
    \begin{align}
    \label{eq:min>max}
        \min_{(s, a) \in \gDE_h \cup \gD^{\oS, 1}_h} \phi_h (s, a) > \max_{(s^\prime, a^\prime) \in \gD^{\oS, 2}_h}  \phi_h (s^\prime, a^\prime).
    \end{align}
    Then we claim that $\theta^\star_h > 0$. Otherwise, it holds that
    \begin{align*}
         \Delta_h (\theta^\star_h) &= \min_{(s, a) \in \gDE_h \cup \gD^{\oS, 1}_h}  \theta^\star_h \phi_h (s, a) - \max_{(s^\prime, a^\prime) \in \gD^{\oS, 2}_h}  \theta^\star_h  \phi_h (s^\prime, a^\prime)
        \\
        &= \theta^\star_h \lp \max_{(s, a) \in \gDE_h \cup \gD^{\oS, 1}_h} \phi_h (s, a) - \min_{(s^\prime, a^\prime) \in \gD^{\oS, 2}_h}  \phi_h (s^\prime, a^\prime)  \rp
        \\
        &< 0.
    \end{align*}
    The last inequality follows $\theta^\star_h < 0$ and $\max_{(s, a) \in \gDE_h \cup \gD^{\oS, 1}_h} \phi_h (s, a) - \min_{(s^\prime, a^\prime) \in \gD^{\oS, 2}_h}  \phi_h (s^\prime, a^\prime) > 0$ due to inequality \eqref{eq:min>max}. Thus, the above inequality conflicts with $\Delta_h (\theta^\star_h) > 0$. The claim that $\theta^\star_h > 0$ is proved and $\bar{\theta}_h \theta^\star_h > 0$.
    \item Case II ($\bar{\theta}_h < 0$). Similarly, we get that 
    \begin{align*}
        \Delta_h (\bar{\theta}_h) &= \min_{(s, a) \in \gDE_h \cup \gD^{\oS, 1}_h}  \bar{\theta}_h  \phi_h (s, a) - \max_{(s^\prime, a^\prime) \in \gD^{\oS, 2}_h}  \bar{\theta}_h  \phi_h (s^\prime, a^\prime)
        \\
        &= \bar{\theta}_h \lp \max_{(s, a) \in \gDE_h \cup \gD^{\oS, 1}_h} \phi_h (s, a) - \min_{(s^\prime, a^\prime) \in \gD^{\oS, 2}_h}  \phi_h (s^\prime, a^\prime)  \rp.
    \end{align*}
    Then $\Delta_h (\bar{\theta}_h) > 0$ implies that
    \begin{align}
    \label{eq:max<min}
        \max_{(s, a) \in \gDE_h \cup \gD^{\oS, 1}_h} \phi_h (s, a) < \min_{(s^\prime, a^\prime) \in \gD^{\oS, 2}_h}  \phi_h (s^\prime, a^\prime).
    \end{align}
    Similar to the analysis in case I, we claim that $\theta^\star_h < 0$. Otherwise, it holds that
    \begin{align*}
         \Delta_h (\theta^\star_h) &= \min_{(s, a) \in \gDE_h \cup \gD^{\oS, 1}_h}  \theta^\star_h \phi_h (s, a) - \max_{(s^\prime, a^\prime) \in \gD^{\oS, 2}_h}  \theta^\star_h  \phi_h (s^\prime, a^\prime)
        \\
        &= \theta^\star_h \lp \min_{(s, a) \in \gDE_h \cup \gD^{\oS, 1}_h} \phi_h (s, a) - \max_{(s^\prime, a^\prime) \in \gD^{\oS, 2}_h}  \phi_h (s^\prime, a^\prime)  \rp
        \\
        &< 0.
    \end{align*}
    The last inequality follows $\theta^\star_h > 0$ and $\min_{(s, a) \in \gDE_h \cup \gD^{\oS, 1}_h} \phi_h (s, a) - \max_{(s^\prime, a^\prime) \in \gD^{\oS, 2}_h}  \phi_h (s^\prime, a^\prime) < 0$ due to inequality \eqref{eq:max<min}. Thus, the above inequality conflicts with $\Delta_h (\theta^\star_h) > 0$. The claim that $\theta^\star_h < 0$ is proved and $\bar{\theta}_h \theta^\star_h > 0$.
\end{itemize}
To summarize, we have proved that $\bar{\theta}_h \theta^\star_h > 0$. Similar to the proof of the sufficient condition, by the convexity of $\gL_h (\theta_h)$ and optimality of $\theta^\star_h$, we can obtain that
\begin{align*}
    \lp  \frac{1}{|\gDS|} \sum_{(s, a) \in \gDS_h} \phi_h (s, a)  - \frac{1}{|\gDE|}  \sum_{(s, a) \in \gDE_h} \phi_h (s, a)   \rp \theta^\star_h < 0.
\end{align*}
This directly implies that
\begin{align*}
    \lp  \frac{1}{|\gDS|} \sum_{(s, a) \in \gDS_h} \phi_h (s, a)  - \frac{1}{|\gDE|}  \sum_{(s, a) \in \gDE_h} \phi_h (s, a)   \rp \bar{\theta}_h < 0,
\end{align*}
which completes the proof of necessity.

\end{proof}
\section{Behavioral Cloning with General Function Approximation}
\label{sec:bc_gfa}

In the main text, the theoretical analysis for BC-based algorithms considers the tabular setting in policy learning where a table function represents the policy. Here we provide an analysis of BC with \emph{general function approximation} in policy learning. Notice that the algorithms considered in this paper (i.e., BC, NBCU and WBCU) can be unified under the framework of maximum likelihood estimation (MLE)\footnote{Among these algorithms, the main difference is the weight function in the MLE objective; see Equations \eqref{eq:bc_opt}, \eqref{eq:bc_union} and \eqref{eq:dwbc}.}. Therefore, the theoretical results in the main text can also be extended to the setting of general function approximation by a similar analysis.

We consider BC with general function approximation, which is the foundation of analyzing NBCU and WBCU. In particular, we assume access to a policy class $\Pi = \{ \pi = (\pi_1, \pi_2, \ldots, \pi_h): \; \pi_h \in \Pi_h, \; \forall h \in [H]  \}$ and $\Pi_h = \{ \pi_h : \gS \rightarrow \Delta (\gA) \}$, where $\pi_h$ could be any function (e.g., neural networks). For simplicity of analysis, we assume that $\Pi$ is a finite policy class. The objective of BC is still 
\begin{align*}
  \pi^{\bc} \in  \max_{\pi \in \Pi} \sum_{h=1}^{H} \sum_{ (s, a) \in \gS \times \gA} \widehat{d^{\expert}_h}(s, a) \log \pi_h(a|s),
\end{align*}
but we do not have the analytic solution as in \cref{eq:pi_bc}.

\begin{thm}
\label{thm:bc_imitation_gap_gfa}
Under \cref{asmp:dataset_collection}. In the general function approximation, additionally assume that $\piE \in \Pi$, we have
\begin{align*}
    \expect \ls V (\piE) - V (\pi^{\bc}) \rs \lesssim \min \lb H, H^2\sqrt{\frac{\log (|\Pi| H N_{\expert})}{N_{\expert}}} \rb
\end{align*}
\end{thm}

Compared with \cref{thm:bc_expert}, we find that the only change in theoretical bound is that $\gO(|\gS|/N_{\expert})$ is replaced with $\gO(\sqrt{\log ( |\Pi|)/ N_{\expert}})$. We note that such a change also holds for other algorithms (e.g., NBCU and WBCU). Therefore, our theoretical implications remains unchanged.

\begin{proof}[Proof of \cref{thm:bc_imitation_gap_gfa}]
We apply the policy difference lemma in \citep{kakade2002pg} and obtain that
\begin{align*}
    V (\piE) - V (\pi^{\bc}) &= \sum_{h=1}^H \expect_{s_h \sim d^{\piE}_h (\cdot)} \ls Q^{\pi^{\bc}}_h (s_h, \piE_h (s_h)) - \sum_{a \in \gA} \pi^{\bc}_h (\cdot|s_h) Q^{\pi^{\bc}}_h (s_h, a)  \rs
    \\
    &\leq \sum_{h=1}^H \expect_{s_h \sim d^{\piE}_h (\cdot)} \ls Q^{\pi^{\bc}}_h (s_h, \piE_h (s_h)) (1-\pi^{\bc}_h (\piE_h (s_h)|s_h))  \rs  
    \\
    &\leq H \sum_{h=1}^H \expect_{s \sim d^{\piE}_h (\cdot)} \ls \expect_{a \sim \pi^{\bc}_h (\cdot|s)} \ls \indict \lb a \not= \piE_h (s) \rb \rs \rs
    \\
    &= H \sum_{h=1}^H \expect_{s_h \sim d^{\expert}_h (\cdot)} \ls \TV \lp \piE_h (\cdot|s_h), \pi^{\bc}_h (\cdot|s_h)  \rp \rs. 
\end{align*}
With Theorem 21 in \citep{agarwal2020flambe}, when $|\gDE| \geq 1$, for any $\delta \in (0, 1)$, with probability at least $1-\delta$ over the randomness within $\gDE$, we have that
\begin{align*}
    \expect_{s_h \sim d^{\expert}_h (\cdot)} \ls \TV^2 \lp \piE_h (\cdot|s_h), \pi^{\bc}_h (\cdot|s_h)  \rp \rs \leq 2 \frac{\log (|\Pi| / \delta)}{|\gDE|}.
\end{align*}
With union bound, with probability at least $1-\delta$, for all $h \in [H]$, it holds that
\begin{align*}
    \expect_{s_h \sim d^{\expert}_h (\cdot)} \ls \TV^2 \lp \piE_h (\cdot|s_h), \pi^{\bc}_h (\cdot|s_h)  \rp \rs \leq 2 \frac{\log (|\Pi| H / \delta)}{|\gDE|}, 
\end{align*}
which implies that 
\begin{align*}
     V (\piE) - V (\pi^{\bc}) &\leq  H \sum_{h=1}^H \expect_{s_h \sim d^{\expert}_h (\cdot)} \ls \TV \lp \piE_h (\cdot|s_h), \pi^{\bc}_h (\cdot|s_h)  \rp \rs
    \\
    &\overset{(a)}{\leq}  H \sum_{h=1}^H \sqrt{ \expect_{s_h \sim d^{\expert}_h (\cdot)} \ls \TV^2 \lp \piE_h (\cdot|s_h), \pi^{\bc}_h (\cdot|s_h)  \rp \rs}
    \\
    &\leq \sqrt{2} H^2 \sqrt{\frac{\log (|\Pi| H / \delta)}{|\gDE|}}.
\end{align*}
Inequality $(a)$ follows Jensen's inequality. Taking expectation over the randomness within $\gDE$ yields that
\begin{align*}
    \expect_{\gDE} \ls V (\piE) - V (\pi^{\bc}) \rs &\leq \delta H + (1-\delta)  \sqrt{2} H^2 \sqrt{\frac{\log (|\Pi| H / \delta)}{|\gDE|}}
    \\
    &\overset{(a)}{=} \frac{H}{2|\gDE|} + \lp 1-\frac{1}{2 |\gDE|} \rp \sqrt{2} H^2 \sqrt{\frac{\log (2|\Pi| H |\gDE|)}{|\gDE|}}
    \\
    &\leq \lp  \sqrt{2} + 1 \rp H^2 \sqrt{\frac{\log (2|\Pi| H |\gDE| )}{|\gDE|}}
    \\
    &\leq 4 H^2 \sqrt{\frac{\log (4|\Pi| H |\gDE| )}{|\gDE|}}. 
\end{align*}
Equation $(a)$ holds due to the choice that $\delta = 1/(2 |\gDE|)$. For $|\gDE| = 0$, we directly have that 
\begin{align*}
    \expect_{\gDE} \ls V (\piE) - V (\pi^{\bc}) \rs \leq H.
\end{align*}
Therefore, for any $|\gDE| \geq 0$, we have that
\begin{align*}
    \expect_{\gDE} \ls V (\piE) - V (\pi^{\bc}) \rs \leq 4 H^2 \sqrt{\frac{\log (4|\Pi| H \max\{ |\gDE|, 1 \} )}{\max\{ |\gDE|, 1\}}}. 
\end{align*}
We consider a real-valued function $f (x) = \log (cx) / x, \; \forall x \geq 1$, where $c = 4|\Pi| H$. Its gradient function is $f^\prime (x) = (1-\log (cx)) / x^2 \leq 0, \; \forall x \geq 1$. Then we know that $f (x)$ is decreasing as $x$ increases. Furthermore, we have that $\max\{ |\gDE|, 1\} \geq (|\gDE| + 1)/2, \; \forall |\gDE| \geq 0$. Then we obtain
\begin{align*}
    \expect_{\gDE} \ls V (\piE) - V (\pi^{\bc}) \rs &\leq 4 H^2 \sqrt{\frac{\log (4|\Pi| H \max\{ |\gDE|, 1 \} )}{\max\{ |\gDE|, 1\}}}
    \\
    &\leq 4 H^2 \sqrt{\frac{2 \log (4|\Pi| H (|\gDE|+ 1) )}{|\gDE|+ 1}}. 
\end{align*}
Taking expectation over the random variable $|\gDE| \sim \text{Bin} (N_{\tot}, \eta)$ yields that
\begin{align*}
    \expect \ls V (\piE) - V (\pi^{\bc}) \rs &\leq 4 H^2 \expect \ls \sqrt{\frac{2 \log (4|\Pi| H (|\gDE|+ 1) )}{|\gDE|+ 1}} \rs
    \\
    &\overset{(a)}{\leq} 4 H^2 \sqrt{\expect \ls \frac{2 \log (4|\Pi| H (|\gDE|+ 1) )}{|\gDE|+ 1} \rs}  .
\end{align*}
Inequality $(a)$ follows Jensen's inequality. We consider the function $g (x) = -x \log (x/c) , \; \forall x \in (0, 1]$, where $c = 4 |\Pi| H$.
\begin{align*}
    g^\prime (x) = - (\log (x/c) + 1) \geq 0, \;  g^{\prime \prime} (x) = - \frac{1}{x} \leq 0, \; \forall x \in (0, 1].
\end{align*}
Thus, $g (x)$ is a concave function. By Jensen's inequality, we have that $\expect [g(x)] \leq g (\expect [x])$. Then we can derive that
\begin{align*}
    \expect \ls V (\piE) - V (\pi^{\bc}) \rs & \leq 4 H^2 \sqrt{\expect \ls \frac{2 \log (4|\Pi| H (|\gDE|+ 1) )}{|\gDE|+ 1} \rs}
    \\
    &= 4 \sqrt{2} H^2 \sqrt{ \expect \ls g \lp \frac{1}{|\gDE| + 1} \rp \rs}
    \\
    &\leq 4 \sqrt{2} H^2 \sqrt{ g \lp \expect \ls  \frac{1}{|\gDE| + 1} \rs \rp}
    \\
    &\overset{(a)}{\leq} 4 \sqrt{2} H^2  \sqrt{ g \lp \frac{1}{N_{\expert}} \rp} 
    \\
    &\leq 4 \sqrt{2} H^2\sqrt{ \frac{ \log (4|\Pi| H N_{\expert} )}{N_{\expert}}}.
\end{align*}
In inequality $(a)$, we use the facts that $g^\prime (x) \geq 0$ and $\expect \ls  1/(|\gDE| + 1) \rs \leq 1/N_{\expert} $ from \cref{lem:binomial_distribution}. We complete the proof.
\end{proof}

\section{Experiments}
\label{appendix:experiments}

\subsection{Experiment Details}
\label{appendix:experiment_details}

\textbf{Dataset Collection.}  We train an online SAC agent \citep{haarnoja2018sac} with 1 million steps using the rlkit codebase\footnote{\url{https://github.com/rail-berkeley/rlkit}}, which is the same as benchmark dataset D4RL\footnote{\url{https://github.com/Farama-Foundation/D4RL}}. We use the deterministic policy as the expert policy, which is common in the literature \citep{ho2016gail} since the deterministic policy usually gives a better performance than the stochastic policy; see \cref{fig:sac} for the training curves of online SAC.

\begin{figure}[htbp]
    \centering
    \includegraphics[width=\linewidth]{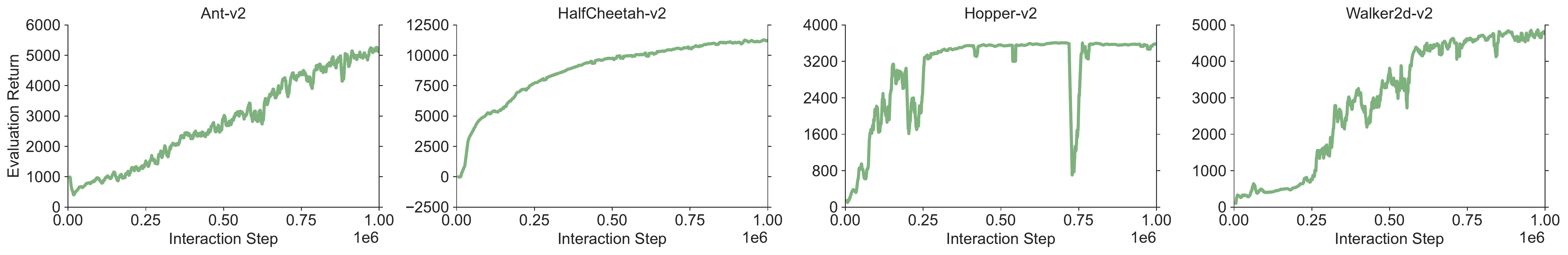}
    \caption{Training curves of online SAC.}
    \label{fig:sac}
\end{figure}

In our experiments, the expert dataset has 1 expert trajectory that is collected by the trained SAC agent. Two tasks differ in the supplementary dataset.

\textbf{Noisy Expert Task.} The supplementary dataset has 10 clean expert trajectories and 5 noisy expert trajectories. For noisy trajectories, the action labels are replaced with random actions (drawn from $[-1, 1]$).

\textbf{Full Replay Task.} The supplementary dataset is from the replay buffer of the online SAC agent, which has 1 million samples (roughly 1000+ trajectories).

\textbf{Algorithm Implementation.} The implementation of DemoDICE is based on the original authors' codebase\footnote{\url{https://github.com/KAIST-AILab/imitation-dice}}. Same as DWBC\footnote{\url{https://github.com/ryanxhr/DWBC}}. We have fine-tuned the hyper-parameters of DemoDICE and DWBC in our experiments but find that the default parameters given by these authors work well. Following \citep{kim2022demodice}, we normalize state observations in the dataset before training for all algorithms.

We use gradient penalty (GP) regularization in training the discriminator of WBCU. Specifically, we add the following loss to the original loss \eqref{eq:discriminator_opt}:
\begin{align*}
    \min_{\theta} \lp \lnorm g(s, a; \theta) \rnorm - 1 \rp^2,
\end{align*}
where $g$ is the gradient of the discriminator $c(s, a; \theta)$.

The implementation of NBCU and WBCU is adapted from DemoDICE's. In particular, both the discriminator and policy networks use the 2-layers MLP with 256 hidden units and ReLU activation. The batch size is 256 and  learning rate (using the Adam optimizer) is $0.0003$ for both the discriminator and policy. The number of training iterations is 1 million. We use $\delta = 0$ and GP=1, unless mentioned. Please refer to our codebase\footnote{\url{https://github.com/liziniu/ILwSD}} for details.

All experiments are run with 5 random seeds.

\subsection{Additional Results}

In this section, we report the exact performance of trained policies for each MuJoCo locomotion control task. Training curves are displayed in \cref{fig:noisy_expert} and \cref{fig:full_replay}.  The evaluation performance of the last 10 iterations is reported in \cref{tab:noisy_expert} and  \cref{tab:full_replay}. The normalized score on a particular environment is calculated by the following formula. 
\begin{align*}
    \text{Normalized Score} = \frac{\text{Algorithm Performance} - \text{Random Policy Performance}}{\text{Expert Policy Performance} - \text{Random Policy Performance}}.
\end{align*}
In \cref{tab:noisy_expert} and \cref{tab:full_replay}, the normalized score is averaged over 4 environments.

Training curves of WBCU and DemoDICE with gradient penalty are displayed in \cref{fig:noisy_expert_wbcu}, \cref{fig:full_replay_wbcu}, \cref{fig:noisy_expert_demo_dice}, and \cref{fig:full_replay_demo_dice}.

\begin{figure}[htbp]
    \centering
    \includegraphics[width=0.45\linewidth]{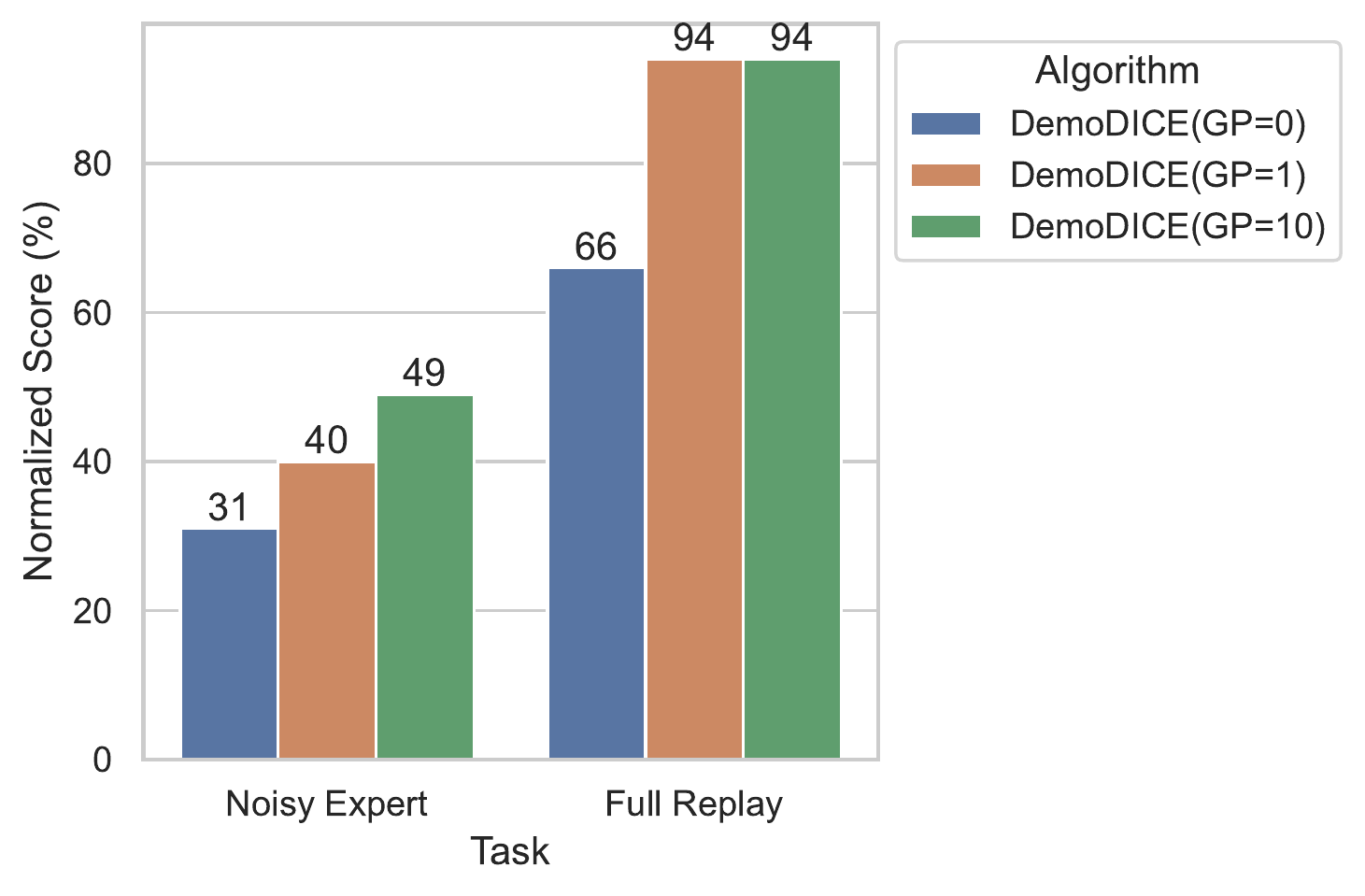}
    \caption{Averaged normalized scores of trained policies of DemoDICE with gradient penalty. Experiments show that the gradient penalty regularization also matters for DemoDICE.}
    \label{fig:demo_dice_gp}
\end{figure}

\begin{table}[htbp]
\centering
\caption{Performance of the trained policies by BC, DemoDICE \citep{kim2022demodice}, DWBC \citep{xu2022discriminator}, NBCU (\cref{algo:nbcu}), and WBCU (\cref{algo:wbcu}) on the noisy-expert task. Numbers correspond to the averaged evaluation return over 5 random seeds ($\pm$ indicate the standard deviation); a lager return means  better performance (same as other tables). }
\label{tab:noisy_expert}
\begin{tabular}{@{}c|llll|l@{}}
\toprule
            & Ant-v2  & HalfCheetah-v2 & Hopper-v2 & Walker2d-v2 & Normalized Score\\ \midrule
Random      &   -325.6        & -280            & -20       & 2     & 0\% \\  
Expert      &   5229          & 11115           & 3589      & 5082   & 100\%      \\ \midrule
BC &  \meanstd{1759}{287}               &  \meanstd{931}{273}            &  \meanstd{2468}{164}      & \meanstd{1738}{311}   &    38\%       \\
\midrule 
DemoDICE(GP=0)    &  \meanstd{1893}{181}               &  \meanstd{2139}{277}           &  \meanstd{1823}{169}                          &  \meanstd{563}{34}                   & 31\%            \\ 
DemoDICE(GP=1)    &  \meanstd{2492}{235}               &  \meanstd{5553}{501}           &  \meanstd{1320}{319}                          &  \meanstd{1153}{220}                   & 40\%            \\ 
DemoDICE(GP=10)    &  \meanstd{2523}{244}               &  \meanstd{6020}{346}           &  \meanstd{1990}{90}                          &  \meanstd{1685}{160}                   & 49\%            \\ 
DWBC    &    \meanstd{3270}{238}             &   \meanstd{5688}{557}                    &  \meanstd{3317}{59}                    &  \meanstd{1985}{175}                 & 62\%              \\ \midrule 
NBCU   &     \meanstd{3259}{159}          &  \meanstd{5561}{539}            &  \meanstd{558}{23}     &      \meanstd{518}{56}  & 35\%      \\
WBCU(GP=0)   &  \meanstd{738}{179}               &  \meanstd{3828}{333}            &  \meanstd{2044}{111}      & \meanstd{477}{81}          & 30\% \\
WBCU(GP=1)   &  \meanstd{2580}{224}               &  \meanstd{8274}{488}            &  \meanstd{3217}{203}      & \meanstd{1932}{329}          & 64\%  \\ 
WBCU(GP=10)   &  \meanstd{2130}{259}               &  \meanstd{7813}{472}            &  \meanstd{2930}{229}      & \meanstd{1510}{476}          & 57\%   \\ \bottomrule
\end{tabular}
\end{table}

\begin{table}[htbp]
\centering
\caption{Performance of the trained policies by BC, DemoDICE \citep{kim2022demodice}, DWBC \citep{xu2022discriminator}, NBCU (\cref{algo:nbcu}), and WBCU (\cref{algo:wbcu}) on the full-replay task. }
\label{tab:full_replay}
\begin{tabular}{@{}c|llll|l@{}}
\toprule
            & Ant-v2  & HalfCheetah-v2 & Hopper-v2 & Walker2d-v2 & Normalized Score \\ \midrule
Random      &   -325.6        & -280            & -20       & 2     & 0\%  \\  
Expert      &   5229          & 11115           & 3589      & 5082   & 100\%      \\ \midrule
BC &  \meanstd{1759}{287}               &  \meanstd{931}{273}            &  \meanstd{2468}{164}      & \meanstd{1738}{311}   &    38\%       \\
\midrule 
DemoDICE(GP=0)    &  \meanstd{4650}{216}               &  \meanstd{9882}{270}           &  \meanstd{39}{0}                          &  \meanstd{4149}{288}                 &  66\%             \\ 
DemoDICE(GP=1)    &  \meanstd{5009}{98}               &  \meanstd{10701}{70}           &  \meanstd{3427}{104}                          &  \meanstd{4560}{146}                 &  94\%             \\ 
DemoDICE(GP=10)    &  \meanstd{5000}{124}               &  \meanstd{10781}{67}           &  \meanstd{3394}{93}                          &  \meanstd{4537}{125}                 &  94\%             \\ 
DWBC    &      \meanstd{2951}{155}           &   \meanstd{1485}{377}                              &     \meanstd{2567}{88}                      &  \meanstd{1572}{225}                & 44\%               \\ \midrule 
NBCU  &   \meanstd{4932}{148} & \meanstd{10566}{86} & \meanstd{3241}{276} & \meanstd{4462}{105} & 92\%  \\
WBCU(GP=0)   &  \meanstd{483}{79}               &  \meanstd{-242}{101}            &  \meanstd{346}{93}      & \meanstd{15}{15}       & 6\% \\
WBCU(GP=1)   &  \meanstd{4935}{108}               &  \meanstd{10729}{74}            &  \meanstd{3390}{132}      & \meanstd{4509}{142}       & 94\%  \\
WBCU(GP=10)   &  \meanstd{4858}{95}               &  \meanstd{10751}{41}            &  \meanstd{3436}{37}      & \meanstd{4403}{88}       & 93\%    \\ \bottomrule
\end{tabular}
\end{table}

\begin{rem}   \label{rem:dataset_coverage}
Readers may realize that NBCU actually is better than BC on the Ant-v2 and HalfCheetah-v2 environments for the noisy-expert task, which seems to contradict our theory. We clarify there is no contradiction. In fact, our theory implies that in the worst case, NBCU is worse than BC. As we have discussed in the main text, state coverage matters for NBCU's performance. For the noisy expert task, we visualize the state coverage in \cref{fig:noisy_expert_dataset_coverage}, where we use Kernel PCA\footnote{\url{https://scikit-learn.org/stable/modules/generated/sklearn.decomposition.KernelPCA.html}. We use the \dquote{poly} kernel.} to project states. In particular, we see that the state coverage is relatively nice on Ant-v2 and HalfCheetah-v2, and somewhat bad on Hopper-v2 and Walker2d-v2. This can help explain the performance difference among these environments in \cref{tab:noisy_expert}.
\end{rem}

\begin{figure}[htbp]
\begin{subfigure}{.45\textwidth}
\centering
\includegraphics[width=0.95\linewidth]{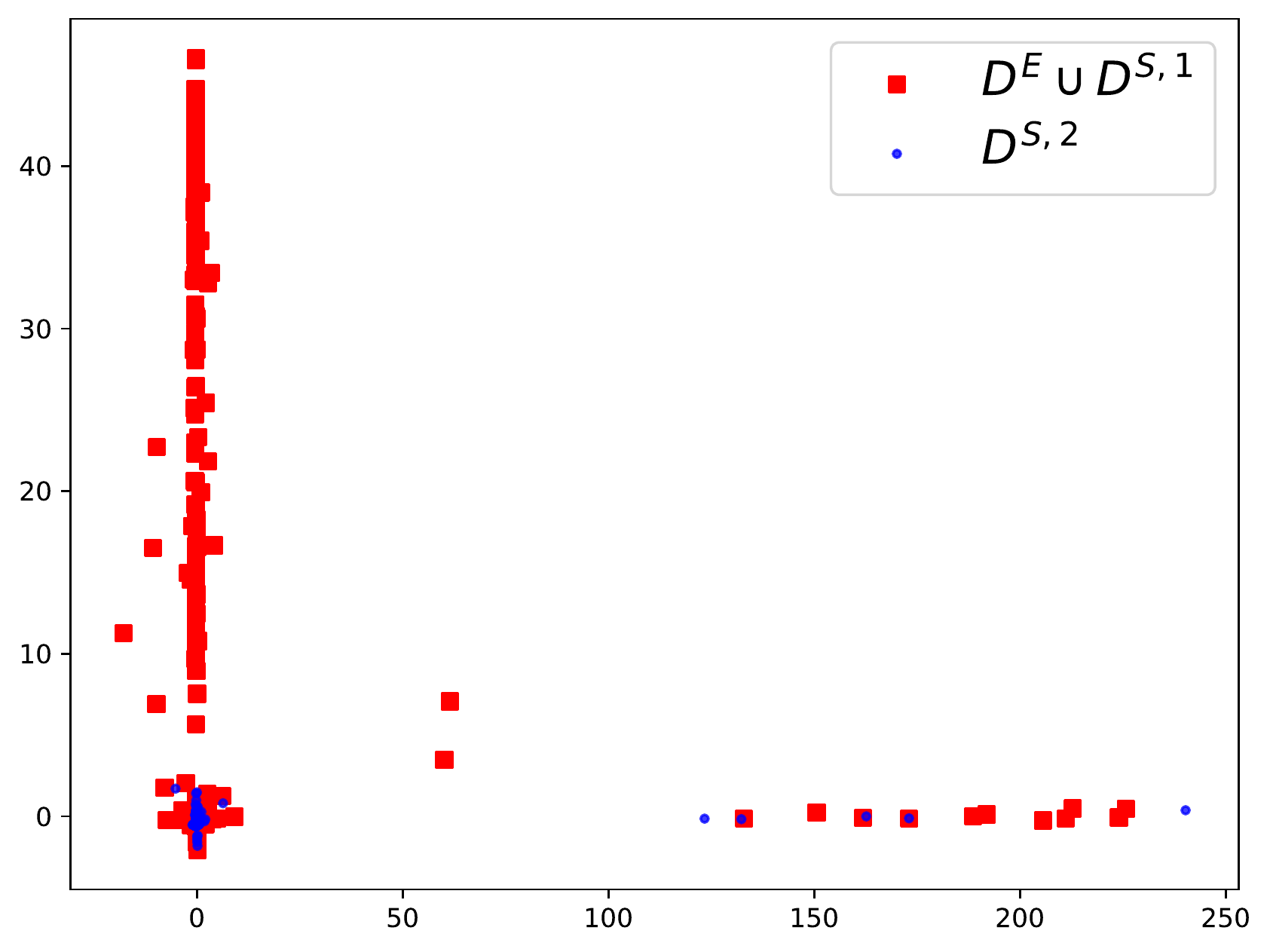}
\caption{Ant-v2.}
\end{subfigure}
\hfill
\begin{subfigure}{.45\textwidth}
\centering
\includegraphics[width=0.95\linewidth]{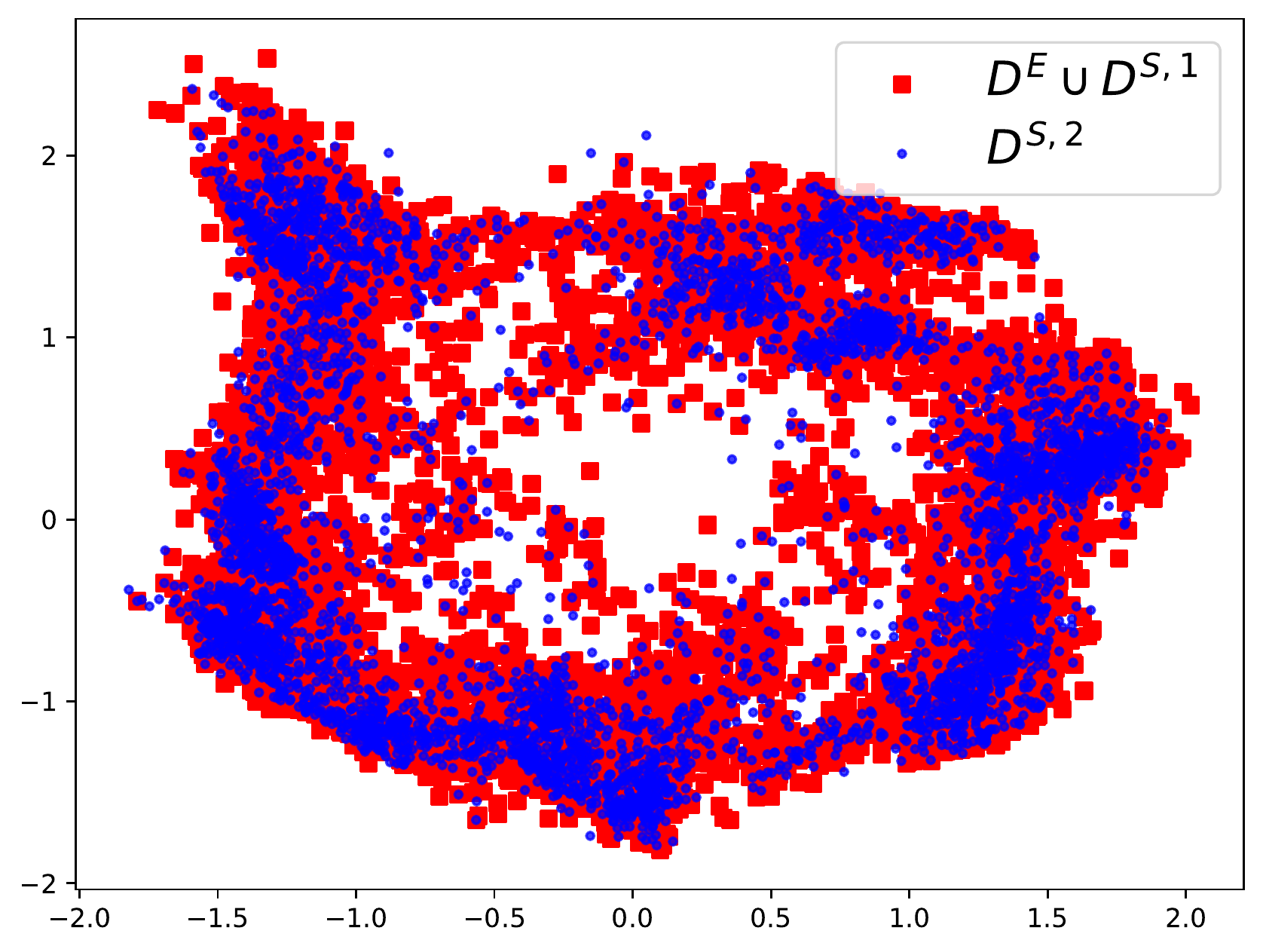}
\caption{Halfcheetah-v2.}
\end{subfigure}
\begin{subfigure}{.45\textwidth}
\centering
\includegraphics[width=0.95\linewidth]{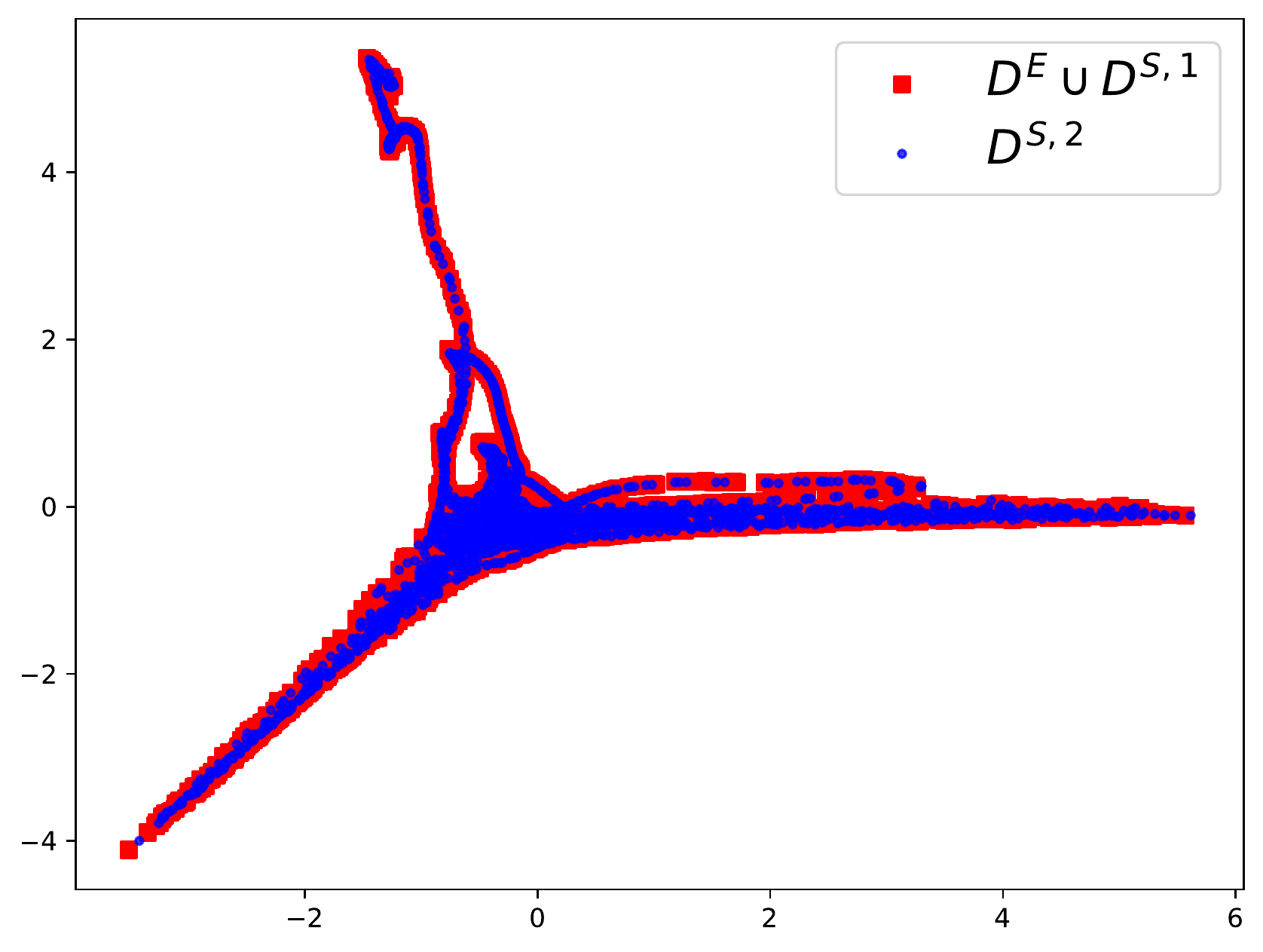}
\caption{Hopper-v2.}
\end{subfigure}
\hfill
\begin{subfigure}{.45\textwidth}
\centering
\includegraphics[width=0.95\linewidth]{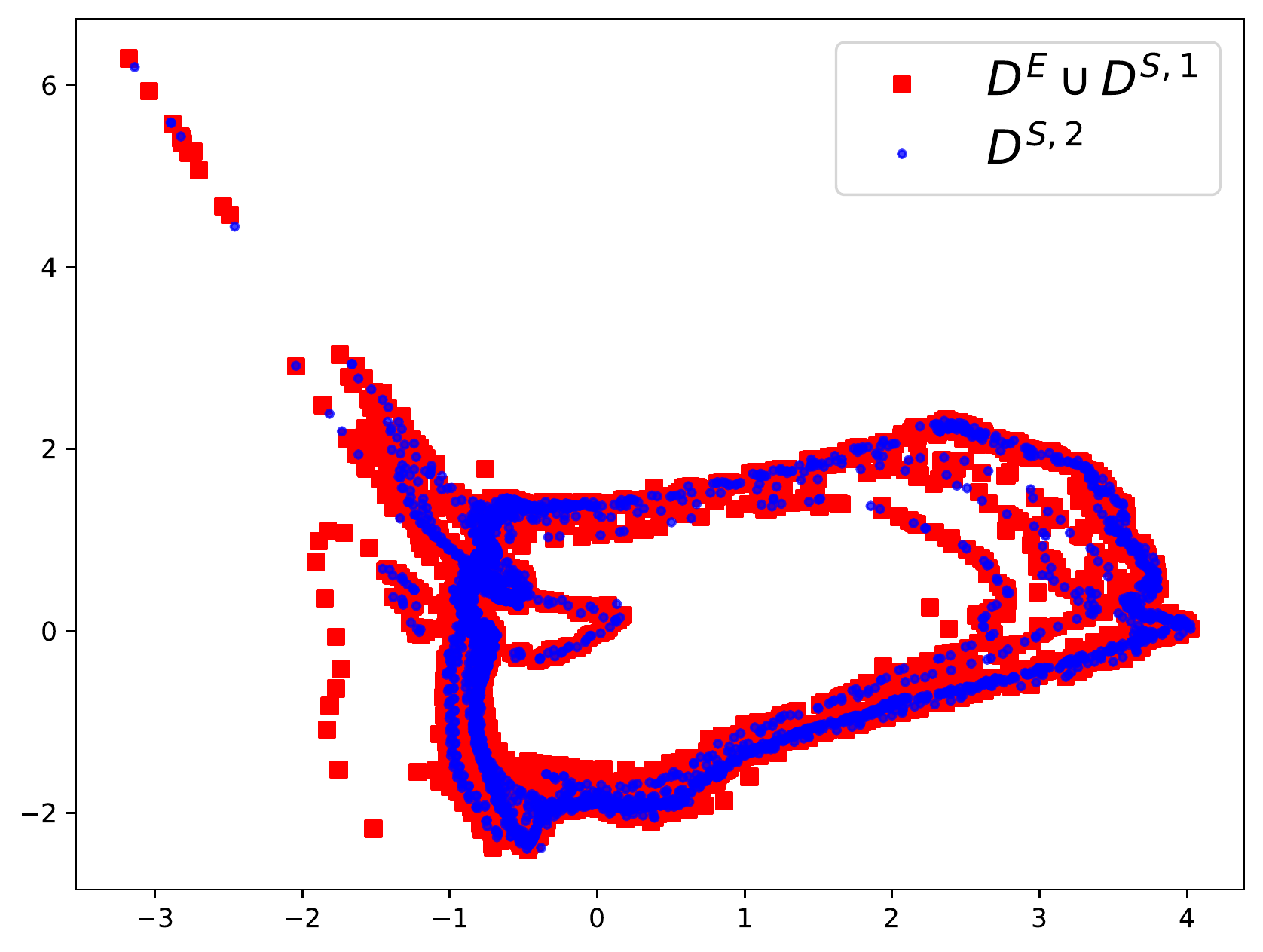}
\caption{Walker2d-v2.}
\end{subfigure}
\caption{Visualization of the state coverage for the noisy expert task. According to the experiment set-up, red points correspond to good samples and blue points correspond to noisy and bad samples. Plots show that the state overlap between two modes of samples is large for Hopper-v2 and Walker2d-v2 and is limited for Ant-v2 and HalfCheetah-v2. Therefore, we can expect NBCU performs well on Ant-v2 and HalfCheetah-v2.}
\label{fig:noisy_expert_dataset_coverage}
\end{figure}

\begin{figure}[htbp]
    \centering
    \includegraphics[width=0.8\linewidth]{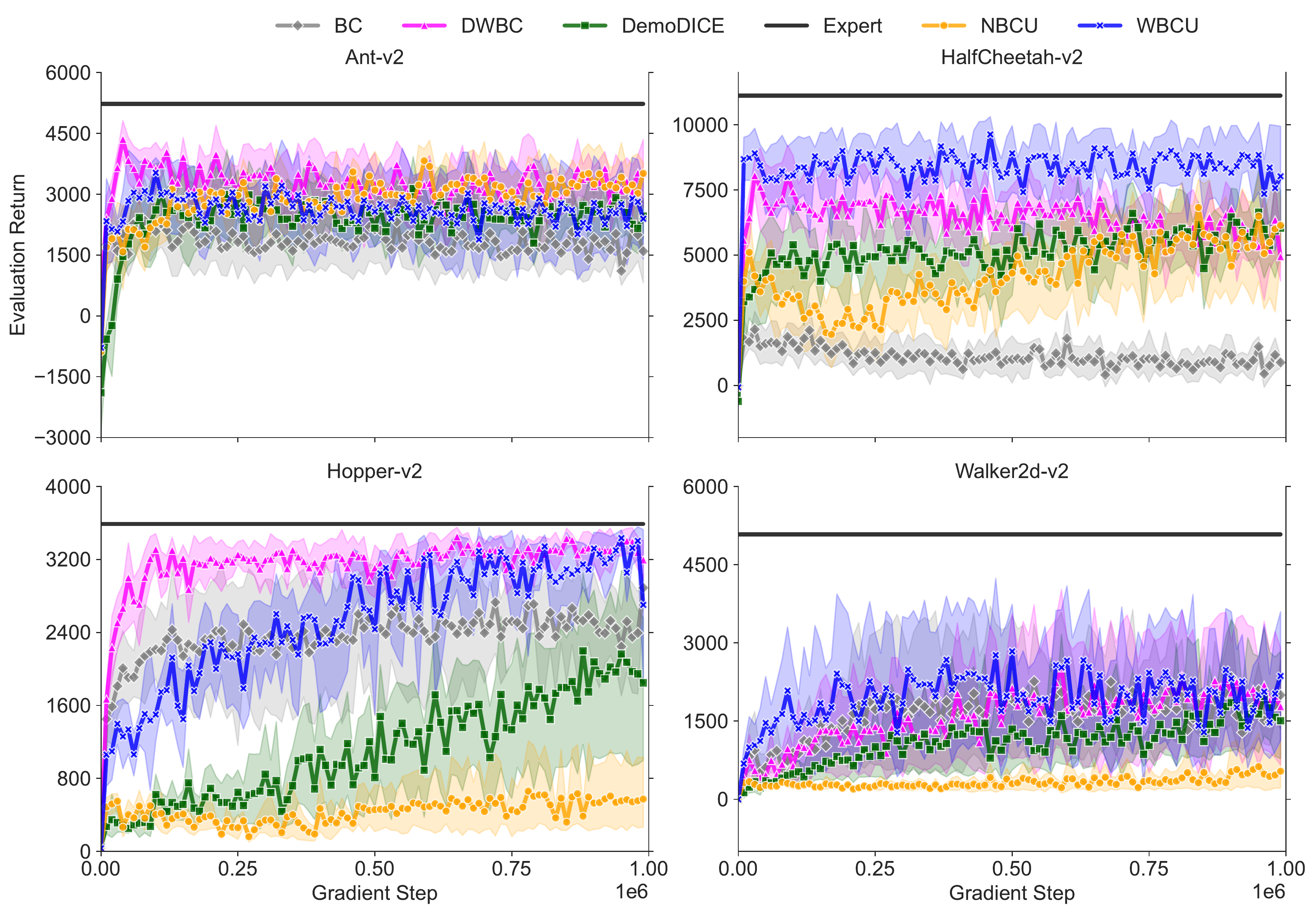}
    \caption{Training curves of BC, DemoDICE \citep{kim2022demodice}, DWBC \citep{xu2022discriminator}, NBCU (\cref{algo:nbcu}), and WBCU (\cref{algo:wbcu}) on the noisy expert task. Solid lines correspond to the mean performance and shaded regions correspond to the 95\% confidence interval. Same as other figures. }
    \label{fig:noisy_expert}
\end{figure}

\begin{figure}[htbp]
    \centering
    \includegraphics[width=0.8\linewidth]{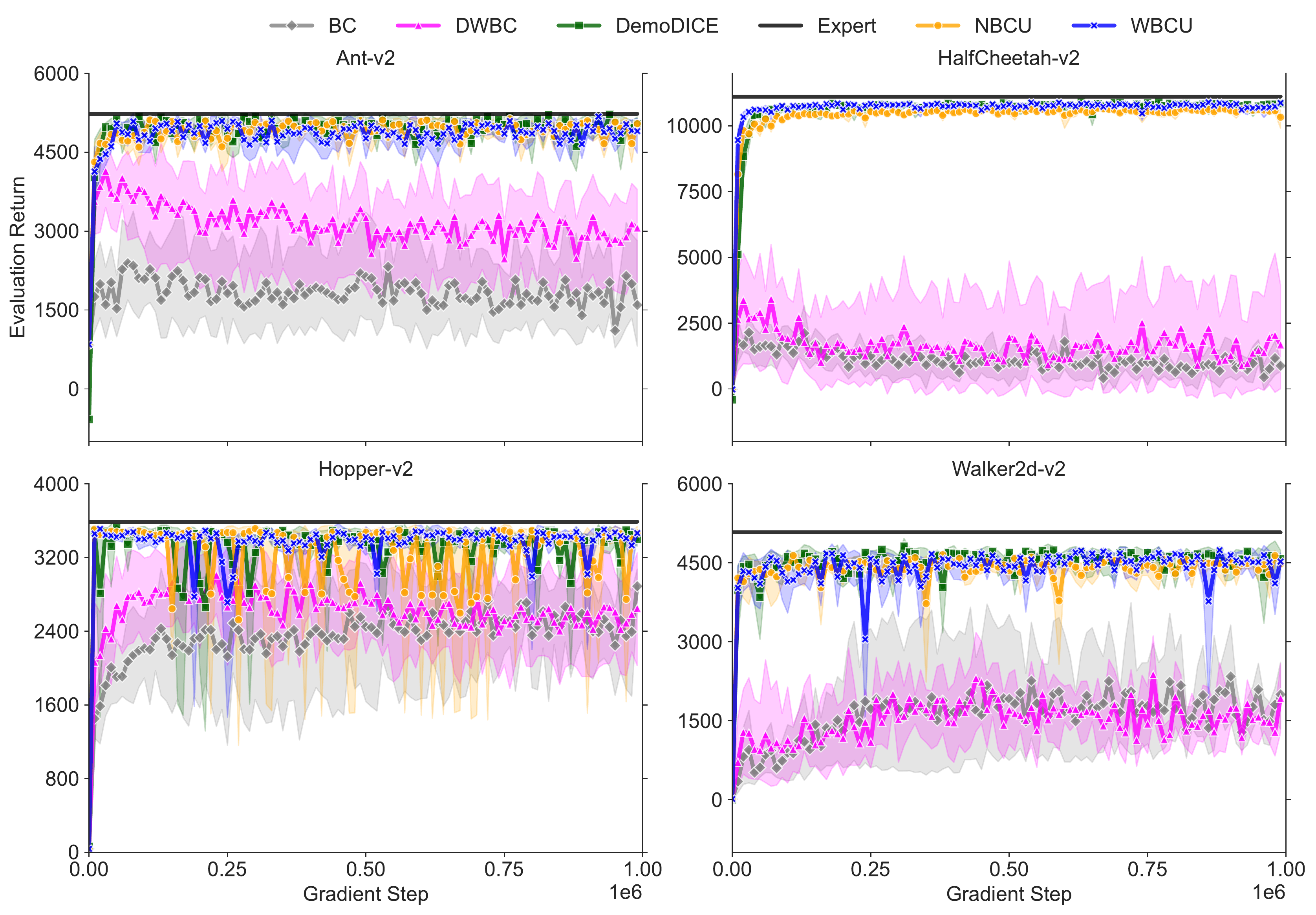}
    \caption{Training curves of BC, DemoDICE \citep{kim2022demodice}, DWBC \citep{xu2022discriminator}, NBCU (\cref{algo:nbcu}), and WBCU (\cref{algo:wbcu}) on the full replay task.}
    \label{fig:full_replay}
\end{figure}

\begin{figure}[htbp]
    \centering
    \includegraphics[width=0.8\linewidth]{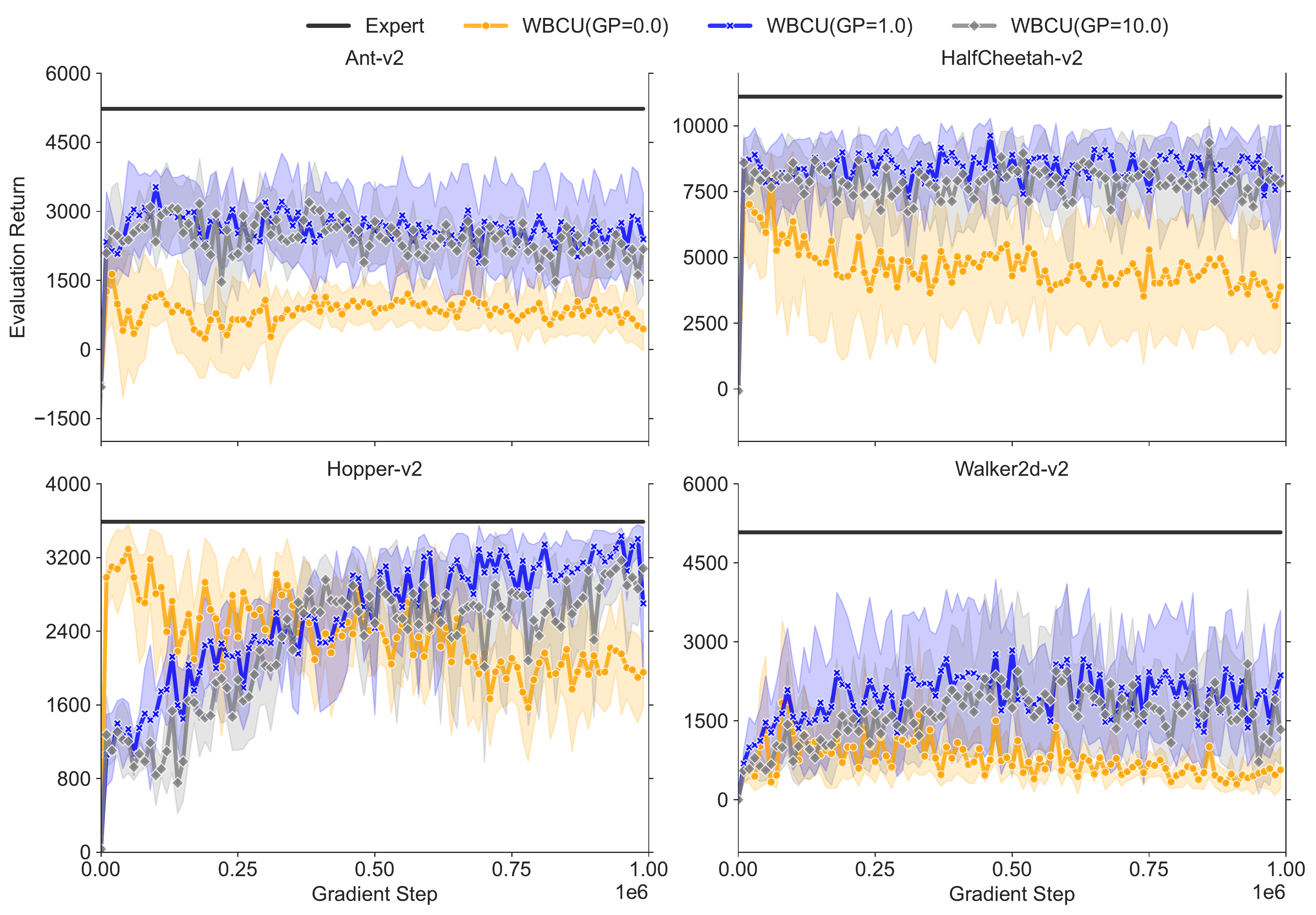}
    \caption{Training curves of WBCU with gradient penalty on the noisy expert task.}
    \label{fig:noisy_expert_wbcu}
\end{figure}

\begin{figure}[htbp]
    \centering
    \includegraphics[width=0.8\linewidth]{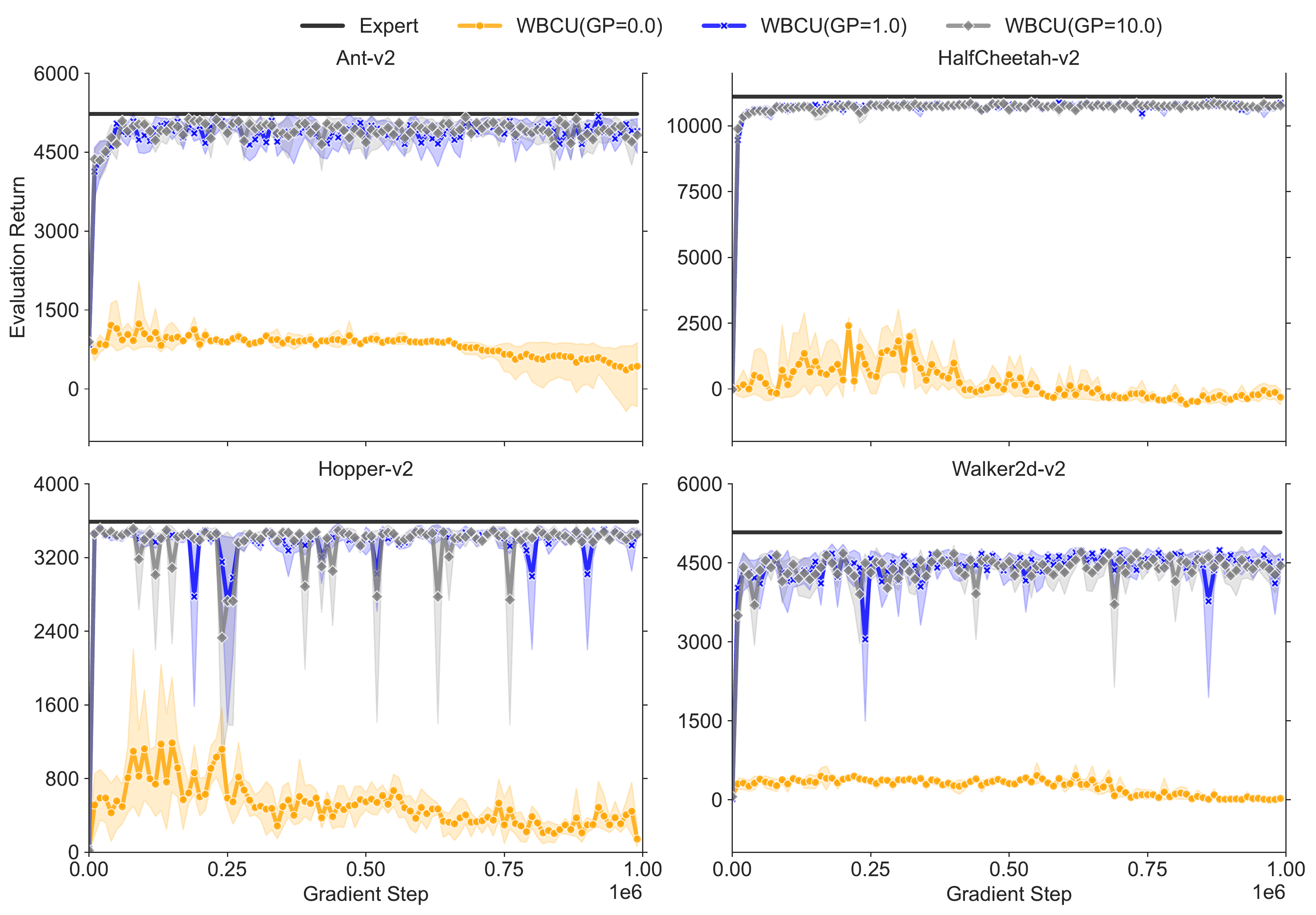}
    \caption{Training curves of WBCU with gradient penalty on the noisy expert task.}
    \label{fig:full_replay_wbcu}
\end{figure}

\begin{figure}[htbp]
    \centering
    \includegraphics[width=0.8\linewidth]{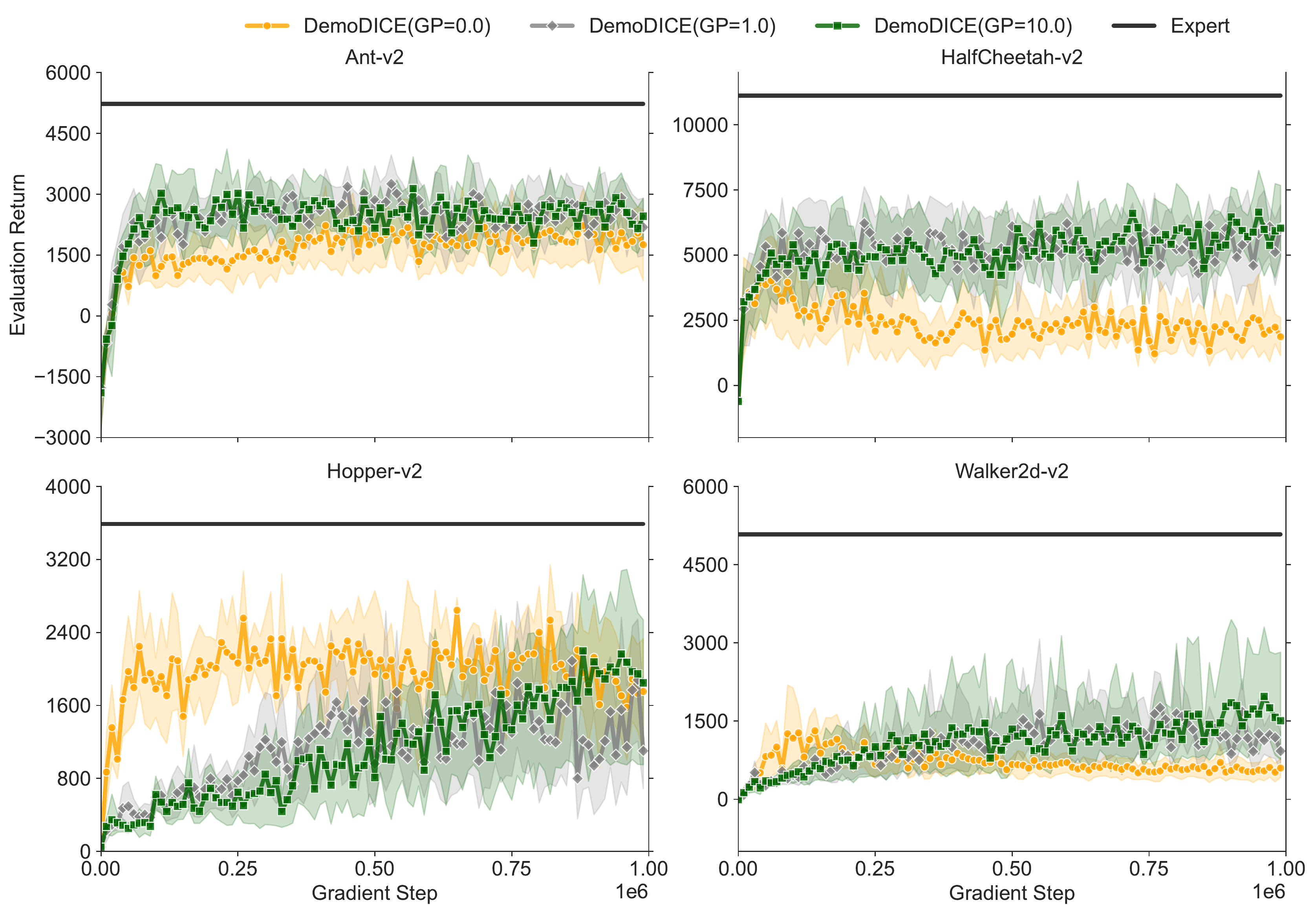}
    \caption{Training curves of DemoDICE with gradient penalty on the full replay task.}
    \label{fig:noisy_expert_demo_dice}
\end{figure}

\begin{figure}[htbp]
    \centering
    \includegraphics[width=0.8\linewidth]{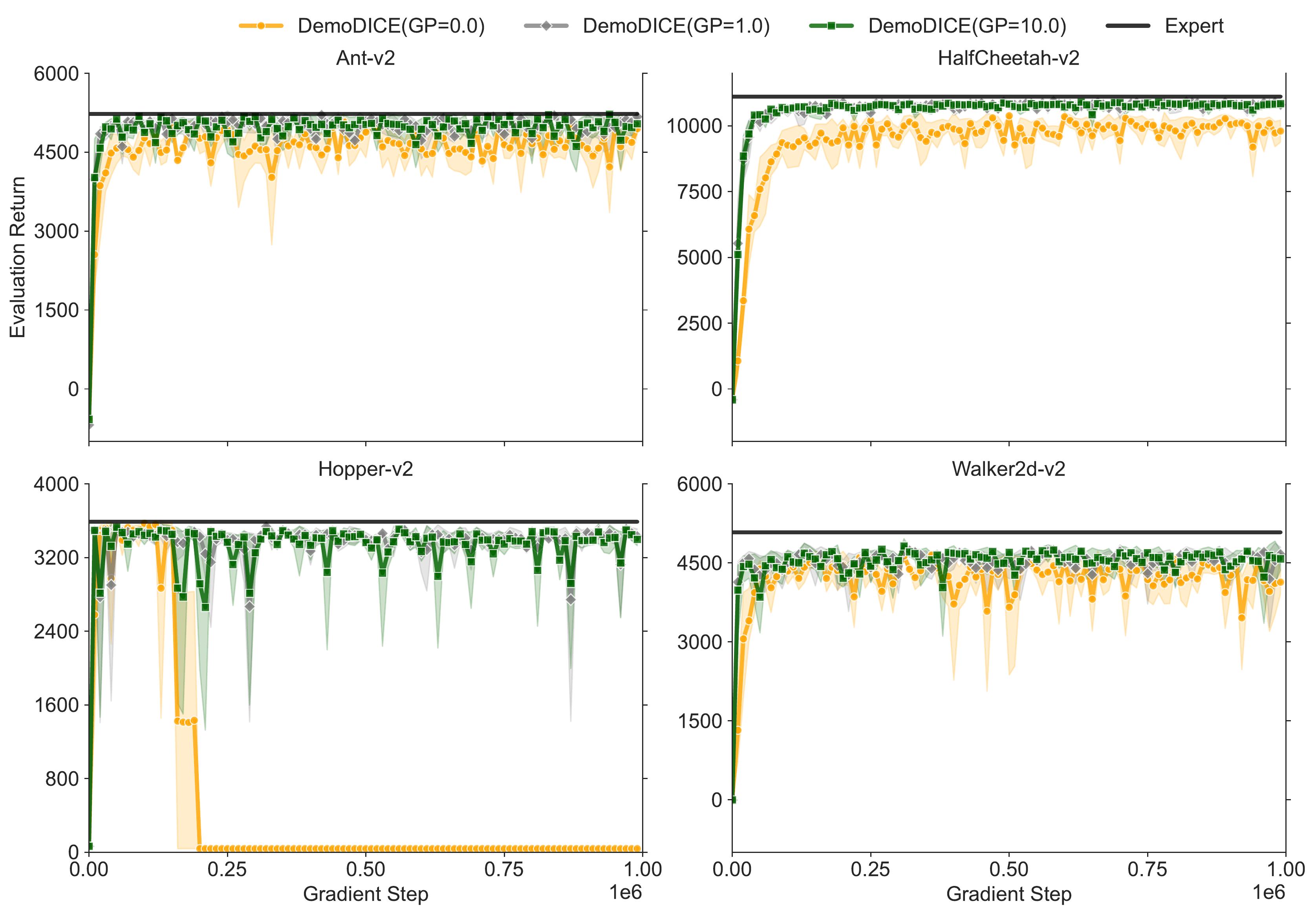}
    \caption{Training curves of DemoDICE with gradient penalty on the full replay task.}
    \label{fig:full_replay_demo_dice}
\end{figure}

\end{document}